\documentclass[runningheads]{llncs}
\usepackage[T1]{fontenc}
\usepackage{graphicx}
\usepackage{longtable}
\usepackage{multirow}

\usepackage{tcolorbox}

\usepackage{mathtools} 

\usepackage{dsfont}     

\usepackage{algorithm}
\usepackage{algorithmicx,eqparbox,array}
\usepackage{algpseudocode}

\algnewcommand\algorithmicswitch{\textbf{switch}}
\algnewcommand\algorithmiccase{\textbf{case}}
\algdef{SE}[SWITCH]{Switch}{EndSwitch}[1]{\algorithmicswitch\ #1\ \algorithmicdo}{\algorithmicend\ \algorithmicswitch}%
\algdef{SE}[CASE]{Case}{EndCase}[1]{\algorithmiccase\ #1}{\algorithmicend\ \algorithmiccase}%
\algtext*{EndSwitch}%
\algtext*{EndCase}%

\usepackage{amsfonts}

\usepackage{hyperref}

\usepackage{cleveref}

\newcommand{\EE}{\mathbb{E}}

\newcommand{\PP}{\mathbb{P}}
\newcommand{\RR}{\mathbb{R}}
\newcommand{\NN}{\mathbb{N}}

\newcommand{\eps}{\varepsilon}
\newcommand{\defas}{\coloneqq}

\newcommand{\calO}{\mathcal{O}}
\newcommand{\SSP}{\mathrm{SSP}}
\newcommand{\Reach}{\mathrm{Reach}}
\newcommand{\ReachBellman}{\mathrm{ReachUpdate}}
\newcommand{\SSPBellman}{\mathrm{SSPUpdate}}
\newcommand{\States}{S}

\newcommand{\state}{s}

\newcommand{\prob}{\delta}
\newcommand{\strategy}{\sigma}
\newcommand{\Strategies}{\Sigma}
\newcommand{\Neighbors}{E}
\newcommand{\Target}{T}

\newcommand{\smallTransition}{\prob_{\min}}
\newcommand{\numStates}{|\States|}
\newcommand{\numNeighbors}{|\Neighbors|}
\newcommand{\numLevels}{k}
\newcommand{\ub}{u_s}
\newcommand{\lb}{l_s}
\newcommand{\guess}{\gamma}

\newcommand{\Paths}{\Omega}

\newcommand{\Plays}{\Omega}
\newcommand{\play}{\omega}

\DeclareMathOperator{\argmax}{argmax}

\DeclareMathOperator{\val}{val}

\begin{document}
	\title{Value Iteration with Guessing for Markov Chains and Markov Decision Processes}
	\titlerunning{Guessing Value Iteration}
	
	\author{
		Krishnendu Chatterjee\inst{1}\orcidID{0000-0002-4561-241X} \and 
		Mahdi JafariRaviz\inst{2}\orcidID{0009-0002-0495-3805} \and 
		Raimundo Saona$^*$\inst{1}\orcidID{0000-0001-5103-038X} \and 
		Jakub Svoboda\inst{1}\orcidID{0000-0002-1419-3267}
	}
	\authorrunning{K. Chatterjee at al.}
	%
	\institute{
		Institute of Science and Technology Austria (ISTA) \url{https://ista.ac.at/} 
		\email{raimundo.saona@gmail.com}
		\email{\{jakub.svoboda, krishnendu.chatterjee\}@ist.ac.at}
		\and
		University of Maryland College Park \url{https://umd.edu} \email{mahdij@umd.edu}
}

\maketitle              
\begin{abstract}
	Two standard models for probabilistic systems are Markov chains (MCs) and Markov decision processes (MDPs).
	Classic objectives for such probabilistic models for control and planning problems are reachability and stochastic shortest path.
	The widely studied algorithmic approach for these problems is the Value Iteration (VI) algorithm which iteratively applies local updates called Bellman updates.
	There are many practical approaches for VI in the literature but they all require exponentially many Bellman updates for MCs in the worst case.
	A preprocessing step is an algorithm that is discrete, graph-theoretical, and requires linear space.
	An important open question is whether, after a polynomial-time preprocessing, VI can be achieved with sub-exponentially many Bellman updates.
	In this work, we present a new approach for VI based on guessing values.
	Our theoretical contributions are twofold.
	First, for MCs, we present an almost-linear-time preprocessing algorithm after which, along with guessing values, VI requires only subexponentially many Bellman updates.
	Second, we present an improved analysis of the speed of convergence of VI for MDPs.
	Finally, we present a practical algorithm for MDPs based on our new approach.
	Experimental results show that our approach provides a considerable improvement over existing VI-based approaches on several benchmark examples from the literature. 
	
\end{abstract}

%
%
%
%

\section{Introduction}
\label{Section: Introduction}

\paragraph{Markov Chains and Markov Decision Processes.}
Markov chains (MCs) and Markov decision processes (MDPs)~\cite{baier2008PrinciplesModelChecking,filar1997CompetitiveMarkovDecision,puterman2014MarkovDecisionProcesses} are widely used mathematical models with applications in various fields including computer science, economics, operations research, and engineering. 
These models study dynamical systems that exhibit stochastic behavior. 
MCs consist of a finite state space and a stochastic transition function.
MDPs extend MCs with non-deterministic choices over actions of a controller that determines the stochastic transition function at each period.

\paragraph{Objectives and Value.} 
Objectives define the payoff to be optimized by the controller.
The classic objectives for MCs and MDPs that arise in control, verification, and planning problems are~\cite{baier2008PrinciplesModelChecking,puterman2014MarkovDecisionProcesses}: 
(i)~weighted reachability objectives and (ii)~stochastic shortest path (SSP) objectives. 
In weighted reachability objectives, we are given a target set of states and weights on them.
Then, the payoff of a path is, if it reaches a target state, the weight of that state, otherwise zero. 
In SSP objectives, we are given a target set and a positive cost at every state.
Then, the cost of a path is, if it reaches a target state, the sum of the cost till a target state is reached, otherwise infinity. 
In MCs, the value is the expectation of the objective.
In MDPs, the value is the optimal expectation over all resolutions of the non-deterministic choices.

\paragraph{Value Iteration.} 
Both MCs and MDPs with the above objectives can be solved in polynomial time \cite[Chapter 10]{baier2008PrinciplesModelChecking} by a reduction to linear programming (LP)~\cite{boyd2004ConvexOptimization}. 
Even though linear programming yields a polynomial-time solution, the algorithm most used in practice is Value Iteration (VI), which is a classic and well-studied algorithmic approach~\cite{chatterjee2008ValueIteration}. 
Given an initial value vector, the VI algorithm iteratively applies local updates, called Bellman updates~\cite{bellman1957markovian}, to approach the value vector. 
The two key advantages of VI over LP are the following.
First, the VI algorithm is simple, elegant, and space-efficient (i.e., it requires linear space) whereas (to the best of our knowledge) linear-space LP-based algorithms for MDPs are not known.
Second, the VI algorithm has symbolic implementations which scales to large state spaces, e.g., representing value vectors as multiterminal binary decision diagrams and applying local updates~\cite{hensel2022ProbabilisticModelChecker,kwiatkowska2011PRISMVerificationProbabilistic}, whereas LP-based algorithms do not. 
Given the advantages of VI over linear programming, VI is the most widely used approach in several tools for probabilistic verification, e.g., PRISM~\cite{kwiatkowsa2012PRISMBenchmarkSuite} and STORM~\cite{hensel2022ProbabilisticModelChecker}.
See~\cite{hartmanns2023PractitionerGuideMDP} for a thorough comparative study. 

\paragraph{Preprocessing.} 
MCs and MDPs are often preprocessed to speed up computation. 
For example: in weighted reachability objectives with binary weights, a preprocessing step is to compute all states with values of either zero or one; 
in SSP objectives, similar computations identify all states with value infinity. 
The desired properties of preprocessings are: 
(a)~discrete and graph-theoretical, i.e., it does not depend on the precise transition probabilities; 
and~(b) linear space, so it ensures space efficiency. 
For various preprocessings for VI, see \cite{brazdil2014VerificationMarkovDecision,haddad2018IntervalIterationAlgorithm,hartmanns2020OptimisticValueIteration,quatmann2018SoundValueIteration}.

\paragraph{Question.} 
Given the practical relevance of VI, multiple VI-based approaches have been proposed in the literature including Interval VI~\cite{haddad2018IntervalIterationAlgorithm,baier2017EnsuringReliabilityYour}, Variance-reduced VI~\cite{sidford2018variance}, Sound VI~\cite{quatmann2018SoundValueIteration}, and Optimistic VI~\cite{hartmanns2020OptimisticValueIteration}. 
However, they all require exponentially many Bellman updates in the number of states in the worst case. 
Thus, a fundamental question is the existence of an efficient preprocessing that can achieve VI with sub-exponentially many Bellman updates.

\paragraph{Our Contribution.}
Our main contributions are the following.
\begin{itemize}
	
	\item \emph{Sub-exponential VI for MCs.} 
	We present an almost linear-time preprocessing algorithm for MCs which, along with guessing values, requires only subexponentially many Bellman updates to approximate the value. 
	Informally, the preprocessing obtains a set of states to guess values which are verified by solving the rest of the MC. 
	The preprocessing to obtain the set of states where values will be guessed is based on breadth-first-search and set cardinality, so it can be symbolically implemented. 
	See \Cref{Result: subexponential MC} for details.
	
	\item \emph{VI convergence for MDPs.} 
	We present a new analysis of the speed of convergence of VI for MDPs. 
	This improves on the previous bounds given in~\cite{haddad2018IntervalIterationAlgorithm} and motivates the use of guessing in MDPs.
	
	\item \emph{Practical approach for MDPs and experimental results.} 
	We present a new practical VI-based approach for MDPs called Guessing VI. 
	We have implemented our approach on STORM~\cite{hensel2022ProbabilisticModelChecker}. 
	We evaluate our approach on the Quantitative Verication Benchmark Set~\cite{hartmanns2019QuantitativeVerificationBenchmark} against all existing VI-based approaches. 
	We consider 474 examples: 
	(i)~in 170 examples all approaches run very fast (less than 100 milliseconds); 
	(ii)~in 135 examples, all approaches perform similarly (the max and min are within 10 \% of each other);
	(iii)~in 83 examples, there is a winner among the previous VI-based approaches over our approach, however, in these examples, the average improvement of each previous approach over ours is at most 1.15 times (i.e., 15\% improvement); and 	
	(iv)~in 86 examples our approach is the fastest, with an average improvement of at least 1.27 times (i.e., 27\% improvement).
	Our experimental results show that our approach provides significant improvement over existing VI-based approaches.
\end{itemize}

\paragraph{Related Works.}

There are three main preprocessings for MDPs used in the literature: (i)~graph connectivity~\cite[Section 3]{even_graph_2012}; (ii)~qualitative reachability~\cite[Algorithms 1-4]{bernardo_automated_2011}; and (iii)~collapsing Maximal end components (MECs)~\cite{chatterjee2011FasterDynamicAlgorithms}.
All these preprocessings run in subquadratic time and are standard in the literature. 
For example, after collapsing MECs, an MDP is usually referred to as contracting or halting~\cite{haddad2018IntervalIterationAlgorithm}. 

For a given horizon, VI provides a strategy that is optimal, and computing such a policy is EXPTIME-complete~\cite{balaji2019ComplexityValueIteration}. 
The most important related works are the VI-based approaches previously presented in the literature, i.e., Interval VI (IVI)~\cite{haddad2018IntervalIterationAlgorithm,baier2017EnsuringReliabilityYour}, Optimistic VI (OVI)~\cite{hartmanns2020OptimisticValueIteration} and Sound VI (SVI)~\cite{quatmann2018SoundValueIteration}, which we have already discussed. 
For MCs and MDPs, there are two main model-checking tools: PRISM~\cite{kwiatkowska2011PRISMVerificationProbabilistic} and STORM~\cite{hensel2022ProbabilisticModelChecker}. 
Our benchmark instances come from the Quantitative Verification Benchmark Set (QVBS)~\cite{hartmanns2019QuantitativeVerificationBenchmark}, which incorporates instances from the PRISM benchmark set~\cite{kwiatkowsa2012PRISMBenchmarkSuite}.

In all works mentioned, the MDPs are fully known a priori. 
When the system is not known a priori, there are approaches based on learning and statistical tests, such as~\cite{brazdil2014VerificationMarkovDecision}. 
Similarly, in all works mentioned, the MDPs are finite. 
When the probabilistic system has uncountably many (continuous) states, sufficient conditions for an anytime approximation of the reachability value are given in~\cite{grover2022AnytimeGuaranteesReachability}. 
While the literature on MCs and MDPs is vast, our work relates to VI-based algorithmic approaches, and hence we restrict our attention to primarily these lines of work.

\section{Preliminaries}
\label{Section: Preliminaries}

\paragraph{Notation.}
For a finite set $X$, we denote the set of probability measures on $X$ by $\Delta(X)$ and the indicator function of $X$ by $\mathds{1}[X]$.
For a natural number $n$, we denote the set $\{1, 2, \ldots, n\}$ by $[n]$.
For two disjoint sets $X$ and $Y$, we denote their disjoint union by $X \sqcup Y$.

\paragraph{Markov Decision Process (MDP).}
An MDP is a tuple \(P = (\States, \Neighbors, \prob)\) with the following properties.
\begin{itemize}
	
	\item 
	\( \States = \States_d \sqcup \States_p \) is a finite set of states partitioned into decision states \( \States_d \) and probabilistic states \( \States_p \).
	
	\item 
	\( \prob \colon \States_p \to \Delta(\States_d) \) is a probability transition function.
	We denote the probability $\prob(\state)(\state')$ by \( \prob(\state, \state') \) and \( \smallTransition \) the smallest positive transition probability.
	
	\item 
	$\Neighbors \subseteq \States \times \States$ is a finite set of edges. 
	For a state $\state \in \States$, we write $\Neighbors(\state) \defas \{ \state' \in \States : (\state, \state') \in \Neighbors \}$ and we require the following properties: (a) for all states $\state \in \States$, the set $\Neighbors(\state)$ is non-empty; and (b) for probabilistic states $\state \in \States_p$, we have that $\Neighbors(\state) = \{ \state' \in \States : \prob(\state, \state') > 0 \}$.
\end{itemize}
Given an initial state $\state_1 \in \States$, the dynamic is as follows.
For every stage $t \ge 1$, if $\state_t \in \States_p$, then the next state $\state_{t + 1}$ is drawn from the distribution $\prob(\state_t)$;
if $\state_t \in \States_d$, then the controller chooses the next state $\state_{t + 1}$ from $\Neighbors(\state)$.
We call the graph $G_P = (\States = \States_p \sqcup \States_d, \Neighbors)$ the labeled graph of $P$ where vertices are labeled as decision or probabilistic.

\paragraph{Strategy.}
A strategy $\strategy$ describes the choice of the controller at each state,
i.e., $\strategy \colon \States_d \to \States$ such that, for all states $\state \in \States_d$,
we have that $\strategy(\state) \in \Neighbors(\state)$.
The set of all strategies is denoted $\Strategies$, which is finite. 
These strategies are referred to as positional in the literature, as opposed to general strategies that depend on the entire history.

\paragraph{Markov Chain (MC).}
A Markov Chain $M$ is a class of MDPs where there is only one strategy. 
Equivalently, for all states $\state \in \States_d$, there is a unique successor, i.e., $|\Neighbors(\state)| = 1$.
Given an MDP $P$, a strategy $\strategy$ induces an MC $P_\strategy$ where each decision state $\state \in \States_d$ transitions deterministically to $\strategy(\state)$.

\paragraph{Plays.}
For an MDP $P$, a play is a legal sequence of states.
Formally, the set of plays is defined by $\Plays \defas \{ \play = (\state_t)_{t \ge 1} : \forall t \ge 1 \quad \state_{t + 1} \in \Neighbors(\state_t) \}$, i.e., the set of infinite paths in the labeled graph of the MDP.
We denote the set of paths starting at $\state$ as $\Paths_{\state}$.

\paragraph{Probabilities.}
An MDP $P$, a strategy $\strategy$ and an initial state $\state$ define a natural dynamic $(S_t)_{t \ge 0}$ over the state space $\States$.
Formally, they define a probability space $(\Plays, F, \PP_{\strategy, \state})$ where the probability measure is the Kolmogorov extension of the natural definition over events defined by a finite sequence of states.
For an MC $M$ and an initial state $\state$, we denote the probability measure by $\PP_{\state}$.

\paragraph{Objectives: Weighted Reachability and Shortest Stochastic Path.}
An objective is a measurable function that assigns a quantity to every play, i.e., a function $\gamma \colon \Plays \to \RR \cup \{ \infty \}$.
A state $\state \in \States$ is absorbing (or sink) if $\Neighbors(\state) = \{ \state \}$.
We consider the following two objectives.
\begin{itemize}
	
	\item {\em Weighted Reachability.} 
	A target set $\Target \subseteq \States$ of absorbing states and an associated weight function $w \colon \Target \to [0, \infty)$ define the weighted reachability objective $\Reach_\Target$ as follows.
	For every play $\play$, if the play reaches state $\state \in \Target$, then $\Reach_\Target(\play) = w(\state)$; if the play never reaches a target state, then $\Reach_\Target(\play) = 0$. 
	We simply write $\Reach$ if $\Target$ is clear from the context.
	
	\item {\em Shortest Stochastic Path (SSP).}	
	A target set $\Target \subseteq \States$ of absorbing states and an associated weight function $w \colon S \to (0, \infty)$ define the stochastic shortest path objective $\SSP_\Target$ as follows.
	For every play $\play$, if the play does not reach a target state, then $\SSP_\Target(\play) = \infty$; and if the play reaches a target state, then $\SSP_\Target(\play)$ is the sum of the weights till a target state is reached including the weight of the reached target state. 
	We simply write $\Reach$ if $\Target$ is clear from the context.
\end{itemize}
For a weight function, we denote $w_{\min} \defas \min\{w(\state) : \state \in \Target\}$ and $w_{\max} \defas \max\{ w(\state) : \state \in \Target\}$.

\paragraph{Value.}
Given an MDP $P$ and an objective, the value of $P$ is the best the controller can guarantee in expectation.
For a play $\play = (\state_t)_{t \ge 1}$, let $t^*(\play) \defas \inf \{ t \ge 1 : s_t \in \Target \}$.
Then, for an initial state $\state$, the value is defined as follows.
\begin{align*}
	\val_{P}(\state; \Reach) 
	&\defas \max_{\strategy} \, \EE_{\strategy, \state}(w(\state_{t^*}) \cdot \mathds{1}[t^* < \infty])  \,, \\
	\val_{P}(\state; \SSP) 
	&\defas \min_{\strategy} \, \EE_{\strategy, \state} \left( \sum_{t = 0}^{t^*} w(\state_{t}) \right) \,.
\end{align*}
If the objective is clear from the context, then we simply write $\val_P(\state)$ or $\val(\state)$.

\paragraph{Value Iteration (VI).}
Value iteration is a classic algorithm that computes the value of an MDP $P$ with either weighted reachability or SSP objectives~\cite{chatterjee2008ValueIteration}.
It simultaneously computes $\val_P(\state)$ for all $\state$.
First, it considers an initial vector $v_1 \colon \States \to \RR$.
At stage $i \ge 1$, it applies a Bellman update operator (described later) on $v_i$ to obtain the next vector $v_{i + 1}$.
The Bellman update operator depends on the objective.
Under a well-chosen initial vector $v_1$, the sequence of vectors $(v_i)_{i \ge 1}$ converges strictly monotonically to the value vector $(\val_P(\state))_{\state \in \States}$.

\paragraph{Bellman Update.}
Given an MDP $P$, for each objective there is a different Bellman update operator.
Consider a state $\state \in \States \setminus \Target$, and a vector $v \colon \States \to \RR$.
For reachability objectives,
\[
\ReachBellman(P,v,\state)
\defas \begin{cases} 
	\max_{\state' \in \Neighbors(\state)} v(\state')
	&\state \in \States_d \\
	\sum_{\state' \in \Neighbors(\state)} \prob(\state, \state') \, v(\state')
	&\state \in \States_p
\end{cases}
\]
For SSP objectives, 
\[
\SSPBellman(P,v,\state)
\defas \begin{cases} 
	w(\state) + \min_{\state' \in \Neighbors(\state)} v(\state')
	&\state \in \States_d \\
	w(\state) + \sum_{\state' \in \Neighbors(\state)} \prob(\state, \state') \, v(\state')
	&\state \in \States_p
\end{cases}
\]
For states in $\Target$, the updates make no changes in the vector.
For simplicity, we denote these two Bellman updates simply $\Call{BellmanUpdate}{}$.
By bounding the output by known lower and upper bounds and starting from a lower or upper bound, iterative applications of $\Call{BellmanUpdate}{}$ generate a monotonic sequence of vectors.

\paragraph{Maximal End Components (MECs).}
Given an MDP, an end component is a set $U$ of states such that, in the corresponding labeled graph,
(a)~$U$ is closed, i.e., for all states $U \cap S_P$ we have $E(s) \subseteq U$;
and (b)~$U$ is a strongly connected component.
Maximal end components (MECs) are end components that are maximal with respect to subset inclusion.

\paragraph{Significance of the Objectives.}
On the one hand, weighted reachability objectives have the following properties:
(a)~they represent the classic reachability objective when every target set is assigned weight~1;
(b)~they correspond to the positive recursive payoff function of Everett~\cite{everett1957RecursiveGames};
and (c)~they naturally arise in many other applications, e.g., in MDPs with long-run average objectives, after computing the value function for every MEC, the problem reduces to weighted reachability objectives~\cite{brzdil2011TwoViewsMultiple}.
While they reduce to classic reachability objectives, we consider them for ease of modeling.
On the other hand, SSP objectives have the useful property that, if the value of a state is known, then we can convert it to a target state with the known value as its weight.

\paragraph{Uniqueness of Fixpoint.}
In general, the Bellman update operator does not have a unique fixpoint.
For reachability objectives, an approach to ensure uniqueness is collapsing MECs, which can be achieved in sub-quadratic time using a discrete graph-theoretical algorithm~\cite{chatterjee2011FasterDynamicAlgorithms}.
For SSP objectives, the Bellman update has a fixpoint if all states can reach the target, which can be checked in linear time. 
Moreover, this fixpoint is unique because weights are strictly greater than zero. 
In the sequel, we consider that Bellman update operators have a unique fixpoint, i.e., for reachability objectives, MECs are already collapsed; and for SSP objectives, the underlying graph is connected.

\paragraph{Error Bounds for VI.}
Given an MDP $P$ and an objective, we call a vector $v$ an \emph{$\eps$-approximation} of the value vector if $\| v - \val \|_\infty \le \eps$. 
For all our objectives, under a well-chosen initial vector $v_1$, the sequence of vectors $(v_i)_{i \ge 1}$ given by VI converges strictly monotonically to the value vector.
Moreover, a convergence bound, i.e., a bound on the distance between the vector $v_i$ and the value vector $\val$, is given as follows. 
Recall that $\smallTransition$ is the smallest positive transition probability.
Then, for all $i \ge 1$,
\[
\| v_{\numStates i} - \val \|_\infty 
\le \left( 1 - \smallTransition^{\numStates} \right)^i \| v_{1} - \val \|_\infty \,.
\]
In particular, for $\eps > 0$, obtaining an $\eps$-approximation of the value is guaranteed after $\mathcal{O} \left( \numStates \log( \| v_{1} - \val \|_\infty / \eps ) / \smallTransition^{\numStates} \right) \ge \mathcal{O} \left(  2^{\numStates} \log(1 / \eps) \right)$ applications of the Bellman operator, i.e., exponentially many in the size of the MDP. 
This bound on the number of updates is necessary for VI in simple examples of MCs.
Therefore, VI requires $\Theta \left(  2^{\numStates} \log(1 / \eps) \right)$, i.e., exponentially many, Bellman updates even for MCs.
This bound applies to all previously defined VI-based approaches.

\paragraph{Preprocessing.} 
The VI algorithm deals with value computation and precise probabilities, and is space efficient because it uses linear space.
Preprocessing steps have been studied to speed up VI in practice. For examples, see~\cite{haddad2018IntervalIterationAlgorithm,hartmanns2020OptimisticValueIteration,quatmann2018SoundValueIteration}.
The desired properties of preprocessing are the following:
(a)~discrete and graph-theoretical so it does not depend on numerical input;
and (b)~linear space, to retain the space efficiency of VI.
A basic open question is the following.

\begin{tcolorbox}
	\paragraph{Break Exponential Barrier for VI.}
	Design a polynomial-time preprocessing such that the number of Bellman updates required to approximate the value in a given MC or MDP is subexponential in number of states.
\end{tcolorbox}

\section{Guessing Value Iteration for MCs}
\label{Section: Binary search in value iteration}

In this section, our main theoretical contribution achieved by our VI-based approach Guessing VI is the following.
\begin{theorem}
	\label{Result: subexponential MC}
	Given an MC $M = (\States, \Neighbors, \prob)$, an objective, and an approximation error $\eps$, we present a preprocessing that runs in linear space and at most $\mathcal{O}((\numStates + |\Neighbors|) \log \numStates)$ steps, so that we require at most $\left( \numStates \log (w_{\max} / \eps ) / \smallTransition \right)^{\calO \left( \sqrt{\numStates} \right)}$ number of Bellman updates to compute an $\eps$-approximation of the value. 
\end{theorem}

\paragraph{Outline.}
First, we introduce Interval VI, the concept of levels in MCs, and recall the speed of convergence of Interval VI for MCs.
Second, we present guessing for MCs for weighted reachability objectives, we recall a result for value computation in simple stochastic games, and present a new result for the approximation of the value.
Third, we present our preprocessing using guessing and the concept of levels in MCs, and prove that it uses linear space and terminates in almost linear time.
Fourth, we present how to use Bellman updates after our preprocessing.
Finally, we present our new VI-based approach called Guessing VI for MCs.

\subsection{Levels and Interval VI}

We define levels for MCs and Interval VI and show their relationship.

\begin{definition}[Levels]
	Consider an MC $M$ and a target set $\Target$. 
	We call \emph{levels} the partitioning of states into $(\ell_0, \ell_1, \ell_2, \dots, \ell_k)$, where: (a)~$\ell_0$ contains $\Target$ and all states that cannot reach $\Target$; and~(b), for all $i \ge 1$, the length of the shortest path in the labeled graph $G_M$ from each state in $\ell_i$ to $\ell_0$ is $i$.
\end{definition}

\begin{remark}
	\label{observation:no_return}
	Given an MC $M$ and a transient state $\state$ in level $\ell_i$, the probability of reaching a target in $i$ steps starting from $\state$ is at least $\smallTransition^i$.
	Also, for every state $s \in \ell_i$ where $i > 0$, there is at least one edge to a state in $\ell_{i-1}$ and no edge to a state in $\ell_j$ for $j<i-1$.
\end{remark}

\paragraph{Interval VI (IVI).}
Interval VI~\cite{haddad2018IntervalIterationAlgorithm} is a VI-based approach that uses lower and upper bounds of the value vector.
IVI starts the iterative updates from the initial vectors $\overline{v_0}$ and $\underline{v_0}$ provided by \Cref{Result: Initial vectors} and the speed of convergence of IVI is given in \Cref{Result: Levels VI}.

\begin{definition}[Initial vectors for value iteration]
	\label{Result: Initial vectors}	
	Consider an MC with $\numLevels$ levels, target $\Target$, and either the weighted reachability or SSP objectives.
	Denote $d(\state, \Target)$ the length of the shortest path in the labeled graph $G_M$ from $\state$ to $T$.
	Then, the \emph{initial vectors for value iteration} consists of an overapproximation $\overline{v_0}$ and underapproximation $\underline{v_0}$ of the value vector.
	For $\state \in \Target$, define $\underline{v_0}(\state) \defas \overline{v_0}(\state) \defas w(\state)$ and, for $\state \in \States \setminus \Target$, 
	\begin{align*}
		\underline{v_0}(\state) 
		&\defas \begin{cases}
			0
			& \text{Reachability} \\
			w_{\min} 
			& \text{SSP} 
		\end{cases}\\
		\overline{v_0}(\state) 
		&\defas \begin{cases}
			w_{\max}
			& \text{Reachability} \\
			w_{\max} (k + 1) / \smallTransition^{\numLevels}
			& \text{SSP} 
		\end{cases}
	\end{align*}
\end{definition}
By \Cref{observation:no_return}, the value vector is between the lower and upper bound for both objectives.

\begin{lemma}%
	\label{Result: Levels VI}%
	Consider an MC $M$ with $\numLevels$ levels and $\smallTransition$ the smallest transition probability.
	For all $t \ge 1$, after $\numLevels \cdot t$ iterations of \Call{IVI}{}, each state in level $i$ has an interval of size at most
	$
	\left(1-\smallTransition^i\right)\left(1-\smallTransition^{\numLevels}\right)^{t-1} \cdot C
	$,
	where $C = w_{\max}$ for reachability and $C = w_{\max} (k + 1) / \smallTransition^{\numLevels}$ for SSP.
\end{lemma}

\begin{proof}[Sketch]
	The proof is a nested induction: first on the number of iterations $t$, then on the level $i$.  
	A level is related to the previous level by the Bellman update since every state has at least one edge going to a state in a previous level. 
	The difference between weighted reachability and SSP objectives is due to the different initial vectors.
	\qed\end{proof}

\subsection{Guessing in Markov Chains}

In this section, we verify a guess on the value of a state as a lower or upper bound through a single Bellman update.
Guesses on the value of a state induce a reduced MC as follows.

\begin{definition}[Reduced Markov Chain]
	Consider an MC $M$, a target set $\Target$, a state $\state \in \States \setminus \Target$, and a quantity $\guess$.
	The \emph{reduced MC}, denoted by $M[\state = \guess]$, is the MC $M$ with target set $\Target \cup \{ \state \}$ where the weight of $\state$ is $\guess$. 
\end{definition}

\begin{remark}[Uniqueness of fixpoints in reduced MCs]
	Consider an MC $M$ where Bellman updates have a unique fixpoint. 
	Then, for all states $\state$ and guesses $\guess > 0$, the reduced MC $M[\state = \guess]$ also has a unique fixpoint.
\end{remark}

The verification of guesses has been established in~\cite[Lemma~3.1]{chatterjee2023FasterAlgorithmTurnbased} in the more general context of stochastic games and now restated for MCs.
\begin{lemma}[{\cite[Lemma~3.1]{chatterjee2023FasterAlgorithmTurnbased}}]%
	\label{Result: from_paper}%
	Consider an MC $M$, a state $\state \in S$, and a guess $\guess$.
	For $f = \val_{M[\state = \guess]}$, let $\guess' \defas \Call{BellmanUpdate}{M, f, s}$.
	Then $\guess' > \guess$ if and only if
	$
	\val_{M}(\state) > \guess\,.
	$
\end{lemma}

By monotonicity of $\Call{BellmanUpdate}{}$, we get the following useful result which has a symmetric statement for upper bounds.
\begin{corollary}%
	\label{Result: from_paper lower bound}%
	Consider an MC $M$, a state $\state \in S$, and a guess $\guess$.
	For a lower bound $f \le \val_{M[\state = \guess]}$, let $\guess' \defas \Call{BellmanUpdate}{M, f, s}$.
	If $\guess' > \guess$, then
	$
	\val_{M}(\state) > \guess\,.
	$
\end{corollary}

As opposed to which focuses on the exact value computation, we focus on the approximation problem.

Exact verification of stochastic systems via guessing values has been established before~\cite{chatterjee2023FasterAlgorithmTurnbased}, but those results are insufficient to solve the approximation problem. 
Indeed, if a guess is very close to the real value, then applying exact verification requires solving the problem at an extremely high precision leading to major time-outs in practice.
Therefore, we require the following approximate verification result.
\begin{lemma}
	\label{lemma:encompassing}
	\label{Result: Encompassing}
	Consider an MC $M$, a state $\state \in S$, and a guess $\guess$.
	For a lower bound $f \le \val_{M[\state = \guess]}$, let $\guess' \defas \Call{BellmanUpdate}{M, f, s}$.
	For all $\eps > 0$, if $\guess' + \smallTransition^{\numStates} \eps > \guess$, then
	$
	\val_{M}(\state) > \guess - \eps \,.
	$
\end{lemma}

\begin{proof}[Sketch]
	Fix $\eps > 0$, the MC $M' \defas M[\state  = \guess - \eps]$ and the function $f' \colon \States \to \RR^+$ defined as
	\[
	f'(\state') = \begin{cases}
		\guess - \eps
		&\state' = \state \\
		f(\state') - \eps \, \PP_{\state'}( \exists t \, \state_t = \state) 
		&\state' \not = \state
	\end{cases}
	\]
	We show that $f' \le \val_{M'}$.
	Then, we argue that $\Call{BellmanUpdate}{M, f', \state} > \guess - \eps$.
	Therefore, applying \Cref{Result: from_paper lower bound} to $M$, $f'$ and $\guess - \eps$, we conclude that $\val_{M}(\state) > \guess-\smallTransition^{-\numStates}\eps$.
	\qed\end{proof}

\subsection{Guessing to Decrease Levels}

By~\Cref{Result: Levels VI}, MCs with few levels can be efficiently solved by IVI. 
We show that, if there are many levels, then the number of levels can be decreased by a factor of $2/3$ by guessing only a few states.
\begin{lemma}
	\label{Lemma: Small Guess}
	Let $M$ be an MC with $\numLevels$ levels.
	There is a level $i \in [\numLevels/3, 2\numLevels/3]$ such that the number of states in level $i$ is at most $\frac{3 \numStates}{ \numLevels}$, i.e.,  $|\ell_i| \le \frac{3\numStates}{\numLevels}$.
\end{lemma}

\begin{proof}
	By contradiction, assume that, for every level $j \in [\numLevels/3, 2\numLevels/3]$,
	we have $|\ell_j| > \frac{3\numStates}{\numLevels}$.
	Then, summing all levels, $\sum_{j} |\ell_j| > \numStates$, which is a contradiction.
	\qed\end{proof}

\Cref{Lemma: Small Guess} immediately indicates a procedure to select states to be guessed while decreasing the number of levels of the resulting MC. 
This procedure is formalized in \Cref{Algorithm: MtG}.
To simplify the notation, we denote $\Call{Guess}{M, I}$ an MC $M$ where the set of states $I$ were transformed into target states.

\begin{algorithm}[t]
	\caption{Decide what states to guess}
	\label{Algorithm: MtG}
	\label{Algorithm: Mark to Guess}
	\begin{algorithmic}[1]
		\Require Markov Chain $M$
		\Ensure Set $I \subseteq \States$ of states to be guessed
		\Procedure{MarkToGuess}{$M$}
		\State $k \gets$ number of levels of $M$
		\Comment{BFS from the target set}
		\If{$k \le \sqrt{\numStates}$}
		\Comment{MC has few levels}
		\State \Return $\emptyset$
		\Comment{No state should be guessed}
		\EndIf
		\State $I \gets $ level between $k/3$ and $2k/3$ with the smallest number of states
		\label{line:thirding} 
		\State $M' \gets \Call{Guess}{M,I}$ 
		\Comment{Mark states in $I$ as guessed}
		\label{line:newMC}
		\State \Return $\Call{MarkToGuess}{M'} \cup I$
		\Comment{Recursive call}
		\EndProcedure
	\end{algorithmic}
\end{algorithm}

\begin{lemma}[Correctness of {$\Call{MarkToGuess}{}$}]
	\label{Lemma: MtG}
	\label{Result: Preprocessing sizes}
	Let $M$ be an MC, \Cref{Algorithm: MtG} finds $I$, such that $|I| \le 9\sqrt{\numStates}$ and the MC $\Call{Guess}{M,I}$ has at most $\sqrt{\numStates}$ levels.
\end{lemma}

\begin{proof}
	From \Cref{Lemma: Small Guess}, we know that, on line~\ref{line:thirding}, the level has at most $\frac{3 \numStates}{k}$ states.
	Moreover, the MC $M'$, defined in line~\ref{line:newMC}, has at most $\frac{2}{3}k$ levels.
	Therefore, \Cref{Algorithm: MtG} outputs a set $I$ such that $|I| \le \sum_{i \ge 0} 3\left(\frac{2}{3}\right)^i \sqrt{\numStates} = 9 \sqrt{\numStates}$.
	
	\qed\end{proof}

\begin{lemma}[Preprocessing complexity]
	\label{Result: Preprocessing complexity}
	Consider an MC $M$. 
	\Cref{Algorithm: Mark to Guess} runs in $\mathcal{O}( (\numStates + |\Neighbors|) \log \numStates )$ steps using linear space.
\end{lemma}

\begin{proof}
	In terms of space, \Cref{Algorithm: Mark to Guess} only needs to perform a BFS from the target set. 
	Therefore, it uses linear space.
	In terms of time, denote $f(\ell)$ the number of steps required to process an MC with $\ell$ levels.
	Then, $f$ satisfies the following recursion.
	Consider an MC with $\ell$ levels.
	The number of steps performed by \Cref{Algorithm: Mark to Guess} includes performing a BFS, constructing the new MC $M'$, and solving an instance with at most $2 \numStates / 3$ levels. 
	Therefore, for $\ell > \sqrt{\numStates}$,
	\[
	f(\ell) \le f \left( \frac{2}{3} \ell \right) + (\numStates + \numNeighbors) + (\numStates + \numNeighbors) \,.
	\]
	Therefore, because $\ell \le \numStates$, we have that $f(\ell) \in \mathcal{O}( (\numStates + \numNeighbors) \log (\numStates) )$.
	In other words, it is almost-linear time.
	
	\qed\end{proof}

\paragraph{Symbolic Computation.}
\Cref{Algorithm: Mark to Guess} is a discrete, graph-theoretical, and linear-space algorithm, i.e., a preprocessing.
Moreover, it only involves a BFS and manipulating the cardinality of sets.
Since all operations can be done symbolically, \Cref{Algorithm: Mark to Guess} can be symbolically implemented.
Indeed, BFS level sets can be obtained by iterative applying the {\em Post} operator that given a set $X$ of states computes the set $Y = \{ s' : \exists s \in X, \, s' \in E(s) \}$~\cite{chatterjeeLowerBoundsSymbolic} and a symbolic computation of the cardinality of sets is presented in~\cite[Section 3]{chatterjee2013SymbolicAlgorithmsQualitative}.

\subsection{Bellman Updates on Guessed MCs}

In this section, we explain how we use Bellman updates to approximate the value.
The preprocessing described by \Cref{Algorithm: Mark to Guess} marks some states to be guessed.
Note that the guesses must be done recursively and Bellman updates are used to verify these guesses.
This idea is formalized in \Cref{Algorithm: Approximate Value of preprocessed MC}.

\begin{algorithm}[t]
	\caption{Approximate Value of preprocessed MC}
	\label{Algorithm: Approximate Value of preprocessed MC}
	\begin{algorithmic}[1]
		\Require Markov Chain $M$, approximation error $\eps$, set of marked states $I$
		\Ensure $l$, $u$ with $max(l - u) < \eps$, the bounds for states' values
		\Procedure{SolveWithGuessingSet}{$M$,$\eps$,$I$}
		\If{$I = \emptyset$} \Comment{No states to be guessed}
		\State \Return $\Call{IVI}{M,\eps}$ \Comment{Solve by IVI}
		\EndIf
		\State $s \in I$ \Comment{Choose a state}
		\State $I' \gets I\setminus s$ \Comment{Update states to be guessed}
		\State $\lb, \ub \gets (\underline{v}_0(\state) , \overline{v}_0(\state) )$ \label{line:start_bounds} \Comment{Initialize bounds}
		\While{$\ub - \lb > \frac{\eps}{2}$} \Comment{Bounds are far apart}
		\State $\guess \gets \frac{\lb + \ub}{2}$ \Comment{Guess the average}
		\State $(l,u) \gets \Call{SolveWithGuessingSet}{M[\state = \guess],\eps\cdot \frac{1}{4}\smallTransition^{\numStates},I'}$
		\label{line: Recursive solve}
		\Comment{Smaller error}
		\If{$\guess < \Call{BellmanUpdate}{l,s}$} \Comment{Guess is small}
		\label{line: Verify guess}
		\State $\lb = \guess$ \label{line:lower}
		\Comment{Update lower bound}
		\ElsIf{$\guess > \Call{BellmanUpdate}{u,s}$} \Comment{Guess is large}
		\State $\ub = \guess$ \label{line:upper}
		\Comment{Update upper bound}
		\Else \Comment{Guess was approximately correct}
		\State $(\lb, \ub)\gets (\guess-\frac{1}{4}\eps, \guess + \frac{1}{4}\eps)$ 
		\label{line:encompass}
		\Comment{Update both bounds}
		\EndIf
		\EndWhile
		\State $(l,u') \gets \Call{SolveWithGuessingSet}{M[\state = \lb], \frac{\eps}{4}, I'}$ 
		\label{line:final_lower}
		\Comment{Use lower bound}
		\State $(l',u) \gets \Call{SolveWithGuessingSet}{M[\state = \ub], \frac{\eps}{4}, I'}$
		\label{line:final_upper}
		\Comment{Use upper bound}
		\State \Return $(l,u)$
		\EndProcedure
	\end{algorithmic}
\end{algorithm}

\begin{lemma}[{Correctness of \Cref{Algorithm: Approximate Value of preprocessed MC}}]%
	\label{Result: Correctness of guessing}%
	Given an MC $M$, an approximation error $\eps$, and a set of states $I \subseteq \States$,
	the procedure given by \Cref{Algorithm: Approximate Value of preprocessed MC}, in other words, $(l,u) = \Call{SolveWithGuessingSet}{M, \eps, I}$, satisfies that $l \le \val_M \le u$ and  $\| u - l \|_\infty \le \eps$.
\end{lemma}

\begin{proof}[Sketch]
	It is enough to show that $\lb \le \val_{M}(\state) \le \ub$ is an invariant of \Cref{Algorithm: Approximate Value of preprocessed MC} and that \Cref{Algorithm: Approximate Value of preprocessed MC} terminates. 
	To do so, we use \Cref{Result: from_paper} and \Cref{lemma:encompassing} to reason about the different cases in each iteration of \Cref{Algorithm: Approximate Value of preprocessed MC}.
	
	\qed\end{proof}

\subsection{Algorithm for MCs}

The final algorithm is a simple concatenation of the preprocessing in \Cref{Algorithm: Mark to Guess} and the use of Bellman updates given by \Cref{Algorithm: Approximate Value of preprocessed MC}.
It takes an MC as an input and runs \Call{MarkToGuess}{} to determine states to be guessed.
With all states guessed, we know the resulting MC has at most $\sqrt{\numStates}$ levels, and we run IVI supplemented by guessing.
We formalize this procedure in \Cref{Algorithm: Approximate Value}.

\begin{algorithm}[t]
	\caption{Approximate Value}
	\label{Algorithm: Approximate Value}
	\begin{algorithmic}[1]
		\Require Markov Chain $M$, approximation error $\eps$
		\Ensure Lower and upper bounds for states' values $l$ and $u$ such that $\|l - u\|_\infty < \eps$
		\Procedure{Solve}{$M$,$\eps$}
		\State $I \gets \Call{MarkToGuess}{M}$
		\State \Return $\Call{SolveWithGuessingSet}{M, \eps, I}$
		\EndProcedure
	\end{algorithmic}
\end{algorithm}

\begin{lemma}[{Complexity of \Cref{Algorithm: Approximate Value}}]%
	\label{Result: subexponential MC Algorith}%
	Consider an MC $M$ and an approximation error $\eps$.
	Let $I$ be the set given by \Cref{Algorithm: Mark to Guess}.
	The number of calls of $\Call{BellmanUpdate}{}$ during the execution of \Cref{Algorithm: Approximate Value of preprocessed MC} is at most $\left( \numStates \log (w_{\max} / \eps ) / \smallTransition \right)^{\calO \left( \sqrt{\numStates} \right)}$.
\end{lemma}

\begin{proof}[Sketch]
	The proof is by induction on the number of states to be guessed, $|I|$. 
	The base case is given by \Cref{Result: Levels VI}, while the inductive step requires using \Cref{Result: Preprocessing sizes}.
	
	\qed\end{proof}

Note that \Cref{Result: Preprocessing complexity} and \Cref{Result: subexponential MC Algorith} prove \Cref{Result: subexponential MC}.

\section{Guessing Value Iteration for MDPs}
\label{Section: MDP extension}

In this section, we discuss the extension of \Cref{Result: subexponential MC} from MCs to MDPs.
Following the ideas for MCs, we partition the states into levels and show that the running time of VI is parametrized by the number of levels.
The following procedure to obtain a level partition has been proposed in~\cite[Proposition 1]{haddad2018IntervalIterationAlgorithm}. 
First, the MDP is reduced by collapsing MECs.
Second, the target states belong to level zero. 
Iteratively, if a probabilistic state $\state$ has a transition to a state $\state'$ with a designated level, then $\state$ belongs to one level higher than $\state'$.
Decision states belong to the highest level it has a transition to.
This level partition leads to a speed of convergence of VI formalized in~\cite[Theorem 2]{haddad2018IntervalIterationAlgorithm} which requires exponentially many Bellman updates even after subexponentially many states have been guessed.
We show that this definition generates more levels than necessary by presenting an alternative level partition and proving a tighter speed of convergence of VI in MDPs.

\begin{definition}[Levels for MDP]
	For MDP $P$ and an optimal strategy $\strategy$, the levels of $P$ given $\strategy$ are the levels of the MC $P_\strategy$.
\end{definition}

This definition of levels depends on an optimal strategy for the MDP. 
Therefore, it corresponds to an ``a posteriori'' bound because it can be computed with information from an optimal strategy (which is equivalent to computing the value vector).
It terms of complexity, both the levels defined in~\cite{haddad2018IntervalIterationAlgorithm} and the levels defined are computed in polynomial time because computing optimal strategies requires only polynomial time.

Our improved speed of convergence for VI relies on the following property.
\begin{property}
	\label{obs:above}
	Bellman updates select an optimal neighbor on decision states.
	Therefore, for an MDP $P$ and an optimal strategy $\strategy$, starting from the same lower bounds, the sequence given by VI on $P$ is always lower bounded by the respective sequence on $P_\strategy$.
\end{property}

We now present the speed of convergence of IVI in the MDP parametrized by levels given by an optimal strategy. 
This result should be compared with \Cref{Result: Levels VI} stated for MCs.
\begin{lemma}
	\label{Result: VI Speed for MDPs}
	Consider an MDP $P$ with $\numLevels$ levels given by an optimal strategy and smallest transition probability $\smallTransition$.
	For all $t \ge 1$, after $\numLevels \cdot t$ iterations of \Call{IVI}{}, for each state $s$ in level $i$, the difference $v(s)-\underline{v_{\numLevels\cdot t}(s)}$ is at most
	$
	\left(1-\smallTransition^i\right)\left(1-\smallTransition^{\numLevels}\right)^{t-1} \cdot C
	$,
	where $C = w_{\max}$ for reachability and $C = w_{\max} (k + 1) / \smallTransition^{\numLevels}$ for SSP.
\end{lemma}

\begin{proof}
	Let $\strategy$ an optimal strategy for $P$.
	In particular, we have that $\val P = \val P_\strategy$.
	Consider $(\underline{v}_i)_{i \ge 1}$ the sequence of lower bounds given by VI on $P_\strategy$.
	By \Cref{Result: Levels VI}, for all $t \ge 1$, we have that $ \| \underline{v}_{\numLevels \cdot t} - \val P\|_\infty \le \left( 1 - \smallTransition^i \right) \left( 1 - \smallTransition^{\numLevels} \right)^{t - 1} \cdot C$.
	By Property~\ref{obs:above}, we conclude the same inequality for the sequence of lower bounds given by VI on $P$.
	\qed\end{proof}

\begin{remark}[Consequence: subexponential preprocessing and Bellman updates for MDPs]
	\label{Result: subexponential for MDPs}
	\Cref{Result: VI Speed for MDPs} implies a procedure to approximate the value that requires subexponential preprocessing time and subexponentially many Bellman updates. 
	Indeed, consider an optimal strategy for the MDP. 
	Then, guess a subset of states of size $\sqrt{\numStates}$ such that the MC induced by the optimal strategy has at most $\sqrt{\numStates}$ levels.
	By \Cref{Result: VI Speed for MDPs}, after subexponentially many Bellman updates we can verify guesses.
	This approach requires either computing an optimal strategy or guessing nondeterministically between all subsets of size $\sqrt{\numStates}$. 
	In particular, this approach is a nondeterministic sub-exponential preprocessing that requires sub-exponentially many Bellman updates to approximate the value vector. 
\end{remark}

While our result for MDPs achieves subexponential preprocessing, improving the preprocessing to polynomial time maintaining subexponentially many Bellman updates remains an open question. 
See \Cref{Section: Level definitions in MDPs} for a discussion.

\section{Practical Guessing VI Algorithm for MDPs}
\label{Section: Implementation}

\Cref{Algorithm: Approximate Value of preprocessed MC} is readily extended to MDPs.
Therefore, we extend \Cref{Algorithm: Approximate Value} from MCs to MDPs by replacing the procedure to obtain a set of states to be guessed in \Cref{Algorithm: Mark to Guess}.
In this section, we explain the major differences between the theoretical procedure of \Cref{Algorithm: Approximate Value} applied to MDPs and our practical implementation.

\paragraph{Early Verification of a Guess.}
Consider \Cref{Algorithm: Approximate Value of preprocessed MC}.
When attempting to verify a guess $\guess$, line~\ref{line: Recursive solve}, it recursively solves an MC with increasing precision.
With the recursive solution, it attempts to verify the guess $\guess$ in line~\ref{line: Verify guess}.
Note that this involves more work than necesasry beacuse, if the Bellman update of state $s$ of lower (upper) bound is above (below) the guess $\guess$, then we can verify the guess as a lower (upper) bound by \Cref{Result: from_paper}.
\Cref{Section: Practical algorithms for MDPs} presents \Cref{Algorithm: Verify a guess} which formalizes this idea and also incorporates an upper limit for the number of Bellman updates used for verification.

\paragraph{Reusing Bounds.}
Consider \Cref{Algorithm: Approximate Value of preprocessed MC}.
When initializing bounds to verify a guess $\guess$ in \Cref{line:start_bounds}, the most conservative bounds are used. 
These bounds can be tightened because, after verifying the guess $\guess$ at state $\state$ as a lower bound, the current vector consists of lower bounds on all states.
Indeed, the values with guess $\guess$ are smaller than the real value.
Similarly, the upper bounds can be reused if the guess is an upper bound.

\paragraph{Picking the Guessed States.}
\Cref{Algorithm: Mark to Guess} prescribes to guess $\mathcal{O}(\sqrt{\numStates})$ states in MCs, as stated in \Cref{Result: Preprocessing sizes}, but guessing fewer states turns out to be faster in practice.
\Cref{Section: Practical algorithms for MDPs} presents \Cref{Algorithm: Pick a state to guess} which formalizes a practical heuristic to choose states to guess.
The idea is to guess a state that, after verification, decreases the current intervals the most.
Therefore, \Cref{Algorithm: Pick a state to guess} starts by weighting each state by the width of its own currently assigned interval.
Another important factor when guessing a state is how fast it will be verified.
In practice, we observed that this is dictated by the influence on its neighbors: the more connected a state is, the faster it will be verified.
Therefore, \Cref{Algorithm: Pick a state to guess} runs a fixed number of steps of a random walk on the labeled graph of the MDP, while accumulating weights given by neighbors.
After this random walk, the state with the highest weight is chosen.

\paragraph{Benefiting from VI.}
Value iteration is a fast algorithm in practice.
Moreover, it is easy to incorporate VI into our algorithm.
Indeed, after verifying the first guessed value as a lower (upper) bound, we obtain lower (upper) bounds for all other states.
Then, to improve the upper (lower) bounds, which have not been updated by the ``guess and verify'' procedure, we can apply Bellman updates.
Our practical approach applies as many Bellman updates as they were used while verifying the guess.

\Cref{Section: Practical algorithms for MDPs} presents \Cref{Algorithm: Practical guessing} which formalizes a practical implementation incorporating all these ideas.
It is a recursive procedure that takes an MDP $P$ and lower and upper bounds. 
It starts by picking carefully a state $\state$ to guess and then attempts to verify the guess $\guess$ through IVI with at most some number of updates.
If the verification of the guess is not successful, then it makes a recursive call asking for the solution of the reduced MDP where $s$ is forced to have value $\guess$ and updates the bounds on the state $\state$ with the result. 

\begin{lemma}[{Correctness of \Cref{Algorithm: Practical guessing}}]%
	\label{Result: Correctness of practical guessing}%
	Given an MDP $P$, the lower bound $l = \underline{v_0}$, the upper bound $u = \overline{v_0}$, an approximation error $\eps$, and parameters $K_1, K_2 \ge 1$, \Cref{Algorithm: Practical guessing}, i.e., $(l, u) =$ \Call{PickVerify}{$P$, $l$, $u$, $\eps$, $K_1$, $K_2$}, terminates and satisfies that $l \le \val_M \le u$ and $\| u - l \|_\infty \le \eps$.
\end{lemma}

\begin{proof}[Sketch]
	Termination is proven by induction on the cardinality of $\States \setminus \Target$ showing that the bounds tighten in each step. 
	Correctness is proven by showing that $l \le \val \le u$ is an invariant. 
	\qed\end{proof}

\section{Experiments}
\label{Section: Experiments (main text)}

In this section, we provide a performance comparison of VI-based approaches.

\paragraph{Algorithms.}
We consider the value approximation of SSP and Reachability MDPs through the use of Bellman updates.
Therefore, we compare the following VI-based approaches.
\begin{itemize}
	\item 
	\textbf{Interval VI (IVI).} 
	Introduced in~\cite{haddad2018IntervalIterationAlgorithm} and extended in~\cite{baier2017EnsuringReliabilityYour}, IVI consists of running simultaneously two VI: one giving an upper bound and the other giving a lower bound, obtaining an anytime algorithm.
	The required preprocessings are as follow. 
	For reachability objectives, MECs are collapsed.
	For SSP objectives, qualitative reachability is solved to obtain a contracting MDP.
	\item 
	\textbf{Optimistic VI (OVI).} 
	Introduced in~\cite{hartmanns2020OptimisticValueIteration}, OVI introduces a candidate vector to speed up VI. 
	Candidate vectors may be validated as lower or upper bounds. 
	If validation fails, then candidates are forgotten.
	\item 
	\textbf{Sound VI (SVI).} 
	Introduced in~\cite{quatmann2018SoundValueIteration}, SVI does not require the a priori computation of starting vectors. 
	It uses lower bounds and VI, while upper bounds are deduced from lower bounds.
	\item 
	\textbf{Guessing VI (GVI).} 
	Introduced in this work, see \Cref{Algorithm: Practical guessing} for details.
	GVI introduces guesses to speed up IVI.
\end{itemize}
All these alternatives have been implemented in the well-known probabilistic model checker STORM~\cite{hensel2022ProbabilisticModelChecker}, including our approach GVI.
The required preprocessings for IVI, OVI, and SVI are as follow.
For reachability objectives, MECs are collapsed, which changes the structure of the MDP and thus is more memory-intensive in model checkers such as STORM.
For SSP objectives, qualitative reachability is solved to obtain a contracting MDP, which only computes the states from which the target is never reached and there is no need to change the structure of the MDP.

Another alternative approach, which we will call Globally Bounded Value Iteartion (GBVI), developed for reachability objectives in Stochastic Games (an extension of MDPs to two opponent controllers) avoids collapsing MECs as a preprocessing~\cite{lahiri_widest_2020}. 
Instead, applied to MDPs, GBVI constructs a weighted graph and solves the widest path problem~\cite{pollack_maximum_1960} on it.
There are subquadratic algorithms to solve the widest path problem, for example, using Fibonacci heaps~\cite{fredman_fibonacci_1984}.
GBVI has been implemented in the probabilistic model checker PRISM~\cite{kwiatkowska2011PRISMVerificationProbabilistic}. 

\paragraph{Benchmarks.}
The Quantitative Verification Benchmark Set (QVBS)~\cite{hartmanns2019QuantitativeVerificationBenchmark}
is an open, freely available, extensive, and collaborative collection of quantitative models to facilitate the development, comparison, and benchmarking of new verification algorithms and tools.
It serves as a benchmark set for the benefit of algorithm and tool developers as well as the foundation of the Quantitative Verification Competition (QComp). QComp is the friendly competition among verification and analysis tools for quantitative formal models.

As opposed to QComp which uses only a curated subset of the benchmark set, we use all instances in QVBS. 
Each instance consists of a model, parameters, and a property.
All properties can be stated as either a reachability or SSP objective.
See \Cref{Section: Experiments} for details on the benchmark instances used in our experiments.


\paragraph{Results.}
We consider all 636 instances contained in the Quantitative Verification Benchmark Set (QVBS)~\cite{hartmanns2019QuantitativeVerificationBenchmark}.
There are 162 instances where some of the algorithms considered timed out (at 600 seconds) or failed.
From these 162 instances, there are 153 in which all algorithms failed or timed out.
In the remaining 11 instances, each algorithm returns an answer as follows:
IVI in 3 instances; 
OVI in 6 instances; 
SVI in 7 instances; 
GVI in 3 instances.
Omitting these 162 instances leaves a total of 474 instances that we analyze.
We grouped the instances as follows.
\begin{itemize}
	\item Group 1: instances where all algorithms are fast, i.e., they take at most 0.1 seconds (170 instances).
	\item Group 2: from the rest, those where the fastest and slowest algorithms are only at most 1.10 times of each other (135 instances).
	\item Group 3: from the rest, there is a winner among the previous VI-based approaches over our approach (83 instances).
	\item Group 4: all other instances not considered before (86 instances).
\end{itemize}

In Group~1, the overall performance of all algorithms are similar (see \Cref{Table: Small times} in \Cref{Section: Experiments}).
In Group~2, the overall performance of the top three algorithms (OVI, SVI, and GVI) are similar where the best and worst differ by at most 1.010 times of each other, whereas IVI is only 1.004 times slower (see \Cref{Table: Indistinguishable times} in \Cref{Section: Experiments}).
In Group~3, the average speedups on the overall performance are as follows compared to GVI: IVI is 0.98 times faster, OVI is 1.03 times faster, and SVI is 1.16 times faster (see \Cref{Table: Worse times} in \Cref{Section: Experiments}).
In Group~4, the average speedup on the overall performance of GVI is 1.33 times faster than IVI, 2.71 times faster than OVI, and 1.28 times faster than SVI (see \Cref{Table: Best Times} in \Cref{Section: Experiments}).

Group~4 is favorable for our algorithm and contains several models including: Randomized consensus protocol~\cite{aspnes1990FastRandomizedConsensus}; Coupon Collectors~\cite{artho_bounded_2016}; Crowds Protocol~\cite{reiter1998CrowdsAnonymityWeb}; IEEE 802.3 CSMA/CD Protocol~\cite{kwiatkowsa2012PRISMBenchmarkSuite}; Dynamic Power Management~\cite{qiu1999StochasticModelingPowermanaged}; EchoRing~\cite{dombrowski_model-checking_2016}; Probabilistic Contract Signing Protocol~\cite{even1985RandomizedProtocolSigning}; Embedded control system~\cite{muppala1994StochasticRewardNetsReliabilityPrediction}; Exploding Blocksworld~\cite{younes2005FirstProbabilisticTrack}; IEEE 1394 FireWire Root Contention Protocol~\cite{stoelinga1999RootContentionIEEE}; Fault-tolerant workstation cluster~\cite{haverkort2000UseModelChecking}; Haddad-Monmege purgatory variant~\cite{haddad2018IntervalIterationAlgorithm}; Cyclic Server Polling System~\cite{ibe_stochastic_1990}.

\Cref{Figure: Walltimes} shows the time (total execution time in the machine, not just CPU time) measured in seconds for every algorithm for each instance in the last two groups.
Instances are ordered by the time achieved by our algorithm.
Note that the $y$-axes are on a logarithmic scale.
See \Cref{Section: Experiments} for detailed tables containing the times of all groups.
See \Cref{Section: Comparative grouping} for a comparison between the distribution of instances in groups 3 and 4 with respect to other algorithms and comparisons of both average and median times.

\begin{figure}[t]
	\centering
	\includegraphics[width=\textwidth]{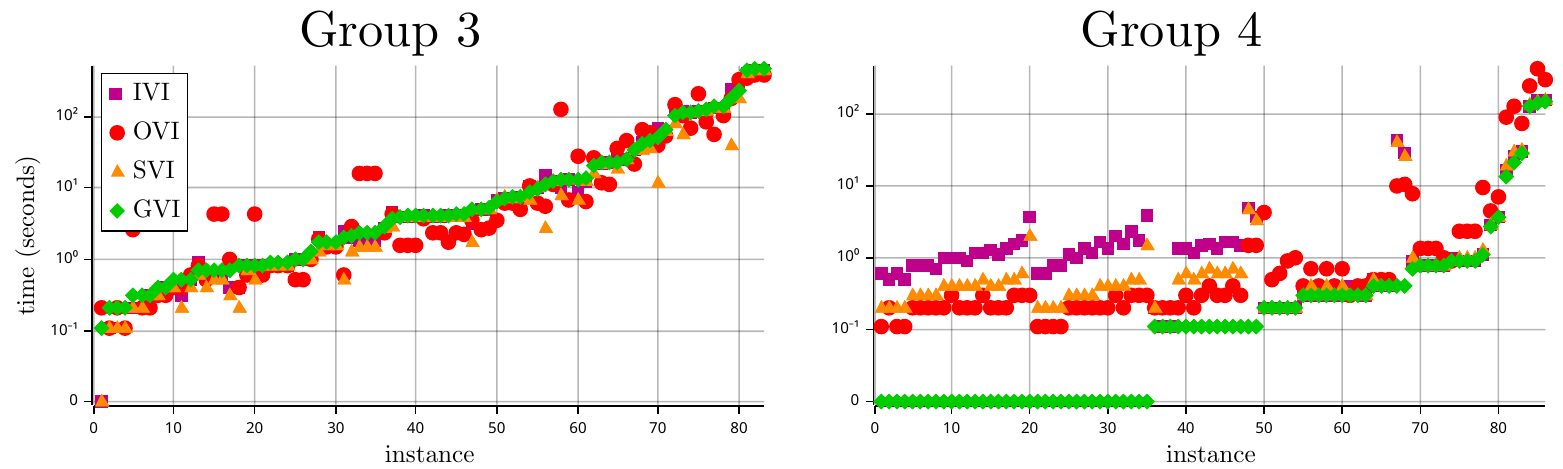}
	\caption{
		Time in seconds of all algorithms over instances in Groups 3 and 4 in increasing order according to GVI and displayed in logarithmic scale.
	}
	\label{Figure: Walltimes}
\end{figure}

\section{Conclusion and Future Works}

In this work, we presented a new approach for VI applied to MCs and MDPs.
For MCs, we proved an almost-linear preprocessing and sub-exponential Bellman updates.
For MDPs, we showed an improved speed of convergence of VI.
It remains an open question whether for MDPs, after polynomial-time preprocessing, VI can be achieved with a sub-exponential number of Bellman updates. 
Finding such an algorithm is an interesting direction for future work.
Our experimental results showed good performance for both MCs and MDPs.
Extending our approach to other models, such as stochastic games, is also a promising path for further research.

\begin{credits}
	\subsubsection{\ackname} 
	This research was partially supported by the ERC CoG 863818 (ForM-SMArt) grant and Austrian Science Fund (FWF) 10.55776/COE12 grant.
	
	\subsubsection{\discintname}
	The authors have no competing interests to declare that are relevant to the content of this article.
\end{credits}

%
%
%
\bibliographystyle{splncs04}
\bibliography{biblio}

\appendix

\section{Discussion of Levels in MDPs}
\label{Section: Level definitions in MDPs}

In this section, we discuss the open question of a polynomial time preprocessing for MDPs which requires only subexponentially many Bellman updates to approximate the value.
For simplicity, we focus on weighted reachability objectives, i.e., the classic reachability objectives and consider MDPs where MECs are collapsed.
Collapsing MECs to a single state is a standard preprocessing technique in MDPs that runs in polynomial time, for details see~\cite{chatterjee2014EfficientDynamicAlgorithms}.

Similar to the levels introduced for MCs, we introduce the following notion of a partition of states for MDPs.
\begin{proposition}[Partition of states]
	\label{Result: Partition of MDPs}
	Consider an MDP $P$. 
	Then, the following holds.
	\begin{itemize}
		\item There is a partition $\{\States_k : k \in \{0, 1, \ldots, K\} \}$ of the set of states such that $\States_0$ contains all states with reachability value zero and one, and, for all $k \in [K]$ and probabilistic state in $\state \in S_k \cap S_p$, we have that the state $\state$ transitions to a state in some previous set of the partition with positive probability, i.e., there exists $\state' \in \bigcup_{k' < k} \States_{k'}$ such that $\prob(\state, \state') > 0$.
		
		\item For all time stage $t \in \NN$ and state $\state \in \States$, there is a strategy $\strategy$ such that the probability of not having reached the target after $t$ stages is upper bounded in terms of the minimum transition probability as follows $\PP_\state^\strategy( \exists t' \le K t : s_{t'} \in T) \le (1 - \smallTransition^K)^t$.
		
		\item For every starting states $\state \in \States$ and strategy $\strategy$, the dynamic eventually reaches the target set, i.e., $\PP_\state^\strategy( \exists t \ge 0 : s_t \in T) = 1$.
	\end{itemize}
\end{proposition}

The most important difference between \Cref{Result: Partition of MDPs} and \cite[Proposition 2]{haddad2018IntervalIterationAlgorithm} is on the second item, where \cite[Proposition 2]{haddad2018IntervalIterationAlgorithm} requires that this item holds for all strategies instead of only some strategy.

The following incorrect result would imply a polynomial time preprocessing for MDPs that requires subexponentially many Bellman updates to approximate the value.
\begin{claim}
	Consider an MDP $P$, and an approximation error $\eps$. Denote $(K + 1)$ the number of partition sets given by \Cref{Result: Partition of MDPs}.
	Then, IVI converges in at most $K \left \lceil \eps / \log(1 - \smallTransition^K) \right \rceil$ steps.
\end{claim}

Let us show how this claim would imply the desired preprocessing.
We can show that \Cref{Result: Preprocessing sizes} and \Cref{Algorithm: Mark to Guess} can readily be extended to MDPs. 
Indeed, \Cref{Algorithm: Mark to Guess} works only on the graph of an MC and the same can be done on the labeled graph of an MDP.
Therefore, we only need to bound the number of Bellman updates required to compute an approximation of the optimal value of MDPs with a partition set of size at most $\sqrt{n}$. 
For the case of MCs, this is given by \Cref{Result: Levels VI}. 
For MDPs, this claim readily implies a similar statement.
Therefore, if the claim were true, then \Cref{Result: subexponential MC Algorith} can easily be extended to MDPs with the corresponding version of \Cref{Algorithm: Approximate Value} for MDPs.
Sadly, the following a simple example shows that this claim is not true.

\begin{example}
	Consider a long chain of decision and probabilistic states where probabilistic states may either advance towards the target set or go back to the decision state that is furthest from the target. 
	This corresponds to a classic example showing an MC where all states have reachability value one, but VI is extremely slow.
	We only change every other state to be a decision state. 
	In this setting, we consider two more states, a sink state and a probabilistic state that goes to the sink and to the target with equal probability. 
	All decision states can also choose to advance to the new probabilistic state that has a reachability value one half. 
	While taking any of these choices is suboptimal, their inclusion in the MDP decreases the number of levels to a constant.
	\Cref{Figure: slow} shows an illustration.
\end{example}

\begin{figure}[t]
	\centering
	\includegraphics[width=\textwidth*7/10]{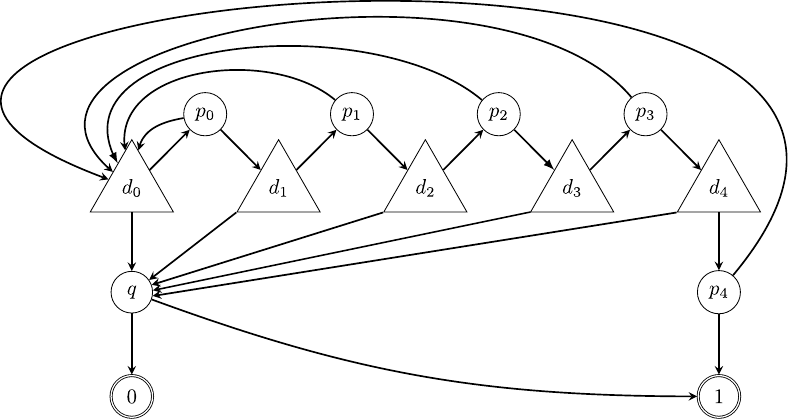}
	\caption{
		MDP with a constant number of levels where IVI requires an ever-growing number of Bellman updates to obtain a $1/4$-approximation of the optimal value for the reachability objective.
		Triangules represeent decision states, circles represent probabilistic states. 
		The transition function assigns equal probability to all outgoing edges of a probabilistic state.
	}
	\label{Figure: slow}
\end{figure}

Following \Cref{Result: Partition of MDPs}, this example admits a partition with four sets: the target and sink form $\States_0$, the two probabilistic states connected to them form the set $\States_1$, all decision states from the set $\States_2$ and the rest of the states form $\States_3$. 
In other words, this is an example of a family of instances with a constant size of the partition set, parametrized by the length of the chain. 

The claim implies that, for reachability objectives, the number of Bellman updates required to obtain an $\eps$-approximation depends only on the size of the partition set, the minimal probability, and $\eps$. 
On the other hand, it is well known that, in this example, the number of Bellman updates required to obtain an $\eps$-approximation of the reachability value grows quickly with the size of the chain. 
Therefore, the claim does not hold.

\section{Practical Algorithms for MDPs}
\label{Section: Practical algorithms for MDPs}

In this section, we present the practical algorithms used for MDPs as pseudocode.

Similar to MCs, a reduced MDP is defined as follows.
\begin{definition}[Reduced MDP]
	Consider an MDP $P$, a target set $\Target$, a state $\state \in \States \setminus \Target$, and a quantity $\guess$.
	The \emph{reduced MDP}, denoted by $P[\state = \guess]$, is the MDP $P$ with target set $\Target \cup \{ \state \}$ where the weight of $\state$ is $\guess$. 
\end{definition}

\begin{algorithm}[H]
	\caption{Verify a guess} 
	\label{Algorithm: Verify a guess}
	\begin{algorithmic}[1]
		\Require MDP $P$, vectors $l$ and $u$, state $\state$, guess $\guess$, approximation error $\eps$ and an iteration limit $K$
		\Ensure Either Bounds$(l', u', k)$, Upper$(u', k)$, or Lower$(l', k)$ or Inconclusive$(l', u')$
		\Procedure{Verify}{$P$, $l$, $u$, $\state$, $\guess$, $K$}
		\For {$k \in [K]$} \Comment {Number of Bellman updates performed}
		\State $(l, u) \gets (\Call{Bellman}{P[\state = \guess], l}, \Call{Bellman}{P[\state = \guess], u})$
		\If { $\| l - u \|_\infty \le \eps \smallTransition^{\numStates} / (2 (1 + \smallTransition^{\numStates}))$ }
		\Comment{$l$ and $u$ are sufficiently close}
		\State $(l, u) \gets ( l - (\eps / 2) \mathds{1}, u + (\eps / 2) \mathds{1} ) )$
		\State \Return Bounds$(l, u, k)$
		\Comment{ Approximations the value }
		\ElsIf { $\guess \le l$ }
		\State \Return Lower$(l, k)$ 
		\Comment{ $\guess$ is a lower bound and so $l$ }
		\ElsIf { $u \le \guess$ }
		\State \Return Upper$(u, k)$ 
		\Comment{ $\guess$ is an upper bound and so $u$ }
		\EndIf
		\EndFor
		\State \Return Inconclusive$(l, u)$ \Comment{ Achieved iteration limit }
		\EndProcedure
	\end{algorithmic}
\end{algorithm}

\begin{algorithm}[H]
	\caption{Pick a state to guess} 
	\label{Algorithm: Pick a state to guess}
	\begin{algorithmic}[1]
		\Require MDP $P$, vectors $l$ and $u$ and maximum number of iterations $K$
		\Ensure A state $s$
		\Procedure{PickState}{$P$, $l$, $u$, $K$}
		\State $w \gets u - l$ \Comment{Initial weight for each state}
		\State $\eta \gets 0$ for every state \Comment{total weight arrived at each state}
		\For {$i \in [K]$}
		\State $w' \gets 0$ for every state
		\ForAll {state $\state$}
		\ForAll {$\state' \in \Neighbors(\state)$}
		\State $w'(s') \gets w'(s') + \begin{cases} 
			\delta(s, s') w(s) 
			& \state \in \States_p \\ 
			\frac{1}{|E(s)|} w(s) 
			& \state \in \States_d \\
		\end{cases}$
		\EndFor
		\EndFor
		\State $w \gets w'$
		\State $\eta \gets \eta + w$
		\EndFor
		\State \Return $\argmax \{ \eta_\state : \state \in \States \setminus \Target \}$ \Comment{State with the highest cumulative weight}
		\EndProcedure
	\end{algorithmic}
\end{algorithm}

\begin{algorithm}[H]
	\caption{Guessing Value iteration: Practical algorithm}
	\label{Algorithm: Practical guessing}
	\begin{algorithmic}[1]
		\Require MDP $P$, vectors $l$, $u$, approximation error $\eps$ and parameters $K_1$ and $K_2$
		\Ensure Lower and upper bounds of the value vector $l$ and $u$ such that $\|l - u\|_\infty < \eps$
		\Procedure{PickVerify}{$P$, $l$, $u$, $\eps$, $K_1$, $K_2$}
		\While {$\|l - u\|_\infty > \eps$}
		\State $\state \gets \Call{PickState}{P, l, u, K_1}$ 
		\Comment{Choose a state to guess}
		\State $\guess \gets \frac{l(s) + u(s)}{2}$
		\Comment{Guess on the selected state}
		\Switch{$\Call{Verify}{P, l, u, \state, \guess, \eps, K_2}$} \Comment{Attempt verification}
		\Case{Bounds$(l', u', k)$}
		\State \Return $l', u'$
		\EndCase
		\Case{Lower$(l', k)$}
		\State $l \gets l'$
		\Comment{Update lower bound}
		\State $k$ times: $u \gets \Call{Bellman}{P, u}$
		\Comment{Advance upper bound $k$ times}
		\EndCase
		\Case{Upper$(u', k)$}
		\State $u \gets u'$
		\Comment{Update upper bound}
		\State $k$ times: $l \gets \Call{Bellman}{P, l}$
		\Comment{Advance lower bound $k$ times}
		\EndCase
		\Case{Inconclusive$(l', u')$}
		\State $(l'', u'') \gets \Call{PickVerify}{P[\state = \guess], l', u', \eps \smallTransition^{\numStates} / ( 4 + 6 \smallTransition^{\numStates})}$
		\label{line: Practical recursive call}
		\Switch{$\Call{Verify}{P, l'', u'', \state, \guess, 0, 1}$} \Comment{One Bellman update}
		\Case{Lower$(l''', k)$}
		\State $l \gets l'''$
		; $k$ times: $u \gets \Call{Bellman}{P, u}$
		\EndCase
		\Case{Upper$(u''', k)$}
		\State $u \gets u'''$
		; $k$ times: $l \gets \Call{Bellman}{P, l}$
		\EndCase
		\Case{else} \Comment{Guess was approximately correct}  \label{line: guess was approximately correct}
		\State $l \gets l'' - (1 + \smallTransition^{\numStates}) \eps / (4 + 6 \smallTransition^{\numStates})$ 
		\State $u \gets u'' + (1 + \smallTransition^{\numStates}) \eps / (4 + 6 \smallTransition^{\numStates})$
		\EndCase
		\EndSwitch
		\EndCase
		\EndSwitch
		\EndWhile
		\State \Return $l, u$
		\EndProcedure
	\end{algorithmic}
\end{algorithm}

\paragraph{Numerical Precision.}
The classic VI requires to handle numbers as small as $\smallTransition^{\numStates}$ and hard examples are easy to construct.
In principle, our VI-based approach also requires to handle this precision.
In practice, since values are guessed, hard examples that require this presicion are not encountered in the benchmark set.
A similar phenomenon has been observed for Optimistic VI~\cite{hartmanns2020OptimisticValueIteration} because it guesses entire value vectors to speed up computations. 
Therefore, using standard floating point representation for numbers works correctly in the benchmark set investigated.
See~\cite{hartmanns2022CorrectProbabilisticModel} for a discussion on the limitations of using floating point numbers in model checking.

\paragraph{Dependency on approximation error.}
Consider the running time of VI-based variants in terms of the approximation parameter $\eps$.
On the one hand, VI requires $\calO \left(  2^{\numStates} \log(1 / \eps) \right)$ Bellman updates even for MCs, i.e., there is a dependecy of the form $\calO \left(  \log(1 / \eps) \right)$.
On the other hand, by \Cref{Result: subexponential MC} we have that Guessing VI requires $\left( \numStates \log (w_{\max} / \eps ) / \smallTransition \right)^{\calO \left( \sqrt{\numStates} \right)}$ Bellman updates to compute an $\eps$-approximation of the value. 
Therefore, a theoretical application of Guessing VI to MDPs would expect a dependency on $\eps$ of $\log (1 / \eps )^{\calO \left( \sqrt{\numStates} \right)}$. 

In contrast, our practical implementation of the core ideas in Guessing VI does not exhibit this worst case behavior in the instances of the studied benchmarks set.
Indeed, we run all instances varying the parameter $\eps$ and observe the following.
Parameterizing $\eps(x) = (10^{6})^{-x}$, the running times we provide in the tables are for the results for $\eps(1)$.
We run all instances in a grid for $x \in [0, 1.5]$ and extrapolate the behavior for $x > 1.5$.

For classic VI, the behavior is at most linear on $x$. 
For the theoretical application of Guessing VI, the behavior is polynomial of the form $x^{\calO \left( \sqrt{\numStates} \right)}$.
For our practical implementation, we observe the following.
Out of all instances solved by our algorithms, in almost all of them the approximation parameter has a negligible affect on the running time.
There are only 2 models where there is a clear increase in running time when $x$ increases (i.e., $\eps$ decreases): 
(i) Randomized Consensus Protocol (\emph{consensus})~\cite{aspnes1990FastRandomizedConsensus}, and
(ii) Flexible Manufacturing System with Repair (\emph{flexible-manufacturing})~\cite{bradley_tagged_2009}.
Under visual inspection, they all exhibit a linear trend.
In other words, in the benchmark set, the running time of our practical implementation of Guessing VI exhibits a dependency on $\eps$ of the form $\calO \left(  \log(1 / \eps) \right)$.

\section{Complete Proofs}
\label{Section: Missing proofs}

\begin{proof}[Proof of {\Cref{Result: Levels VI}}]
	Consider an MC $M$ with $\numLevels$ levels and $\smallTransition$ the smallest transition probability.
	We prove that, for all $t \ge 1$, after $\numLevels \cdot t$ iterations of \Call{IVI}{}, each state in level $i$ has an interval of size at most $\left(1-\smallTransition^i\right)\left(1-\smallTransition^{\numLevels}\right)^{t-1} \cdot C$, where $C = w_{\max}$ for reachability and $C = w_{\max} (k + 1) / \smallTransition^{\numLevels}$ for SSP.
	Since from \Cref{Result: Initial vectors}, the interval has exactly the maximal size, this implies the monotonicity.
	
	The proof is by induction on $t$.
	Consider the base case $t = 1$. 
	We proceed by induction on the level $i$.
	Consider the base case $i = 0$. 
	At level zero, all states are targets and their intervals have zero length since they are assigned the value, so the statement is true.
	Consider the inductive case on $i$.
	By inductive hypothesis on $i$, we have that the intervals for each state in level $i' < i$ have size at most
	\[
	\left(1-\smallTransition^{i'}\right) \cdot C \,.
	\]
	We must prove that this also holds for states in level $i$ too.
	
	Consider the vectors $\underline{v}_{k}$ and $\overline{v}_{k}$. 
	We have that, for each state $s \in \ell_i$,
	\begin{align*}
		|| \underline{v}_{kt}(s) &- \overline{v}_{kt}(s) || \\
		&= \left\| \sum_{\state' \in \Neighbors(\state)} \prob(\state, \state')  
		(\underline{v}_{kt - 1}(s) - \overline{v}_{kt - 1}(s) ) \right\|
		&(\text{def. Bellman update} )\\
		&\le \smallTransition \left( 1 - \smallTransition^{i - 1} \right) \cdot C \\
		&\qquad + (1 - \smallTransition) \left( 1 - \smallTransition^k \right) \cdot C 
		&(\text{induction hypothesis}) \\
		&\le \smallTransition \left( 1 - \smallTransition^{i - 1} \right) \cdot C \\
		&\qquad + (1 - \smallTransition) \cdot C 
		&(\text{omitting one factor}) \\
		&= \left( \smallTransition \left( 1 - \smallTransition^{i - 1} \right) + (1 - \smallTransition) \right) \cdot C 
		&(\text{factorizing}) \\
		&= \left( 1 - \smallTransition^{i} \right) \cdot C \,.
		&(\text{adding})
	\end{align*}
	This completes the induction proof over $i$ for $t = 1$. 
	
	Consider the inductive case on $t$. 
	Similarly, we proceed by induction on $i$. 
	The base case $i = 0$ is trivial, so consider the inductive case on $i$.
	By induction hypothesis on $t$, we know that for all $0 \le t' \le t$, after $\numLevels \cdot t'$ iterations of \Call{IVI}{}, the intervals for a state in each level $i'$ have size at most
	\[
	\left(1-\smallTransition^{i'}\right)\left(1-\smallTransition^{\numLevels}\right)^{t'-1} \cdot C \,.
	\]
	By induction hypothesis on $i$, we know that, after $\numLevels \cdot t$ iterations of \Call{IVI}{}, the intervals for a state in level $i' < i$ have size at most
	\[
	\left(1-\smallTransition^{i'}\right)\left(1-\smallTransition^{\numLevels}\right)^{t} \cdot C \,.
	\]
	We must prove that this also holds for states in level $i$.
	
	Consider the vectors $\underline{v}_{kt}$ and $\overline{v}_{kt}$, we have that, for each state $s \in \ell_i$,
	\begin{align*}
		|| \underline{v}_{kt}(s) &- \overline{v}_{kt}(s) || \\
		&= \left\| \sum_{\state' \in \Neighbors(\state)} \prob(\state, \state')  
		(\underline{v}_{kt - 1}(s) - \overline{v}_{kt - 1}(s) ) \right\|
		&(\text{def. Bellman update} )\\
		&\le \smallTransition \left(1-\smallTransition^{i - 1}\right)\left(1-\smallTransition^{\numLevels}\right)^{t} \cdot C \\
		&\qquad + (1 - \smallTransition) \left(1-\smallTransition^k\right)\left(1-\smallTransition^{\numLevels}\right)^{t-1} \cdot C 
		&(\text{induction hypothesis}) \\
		&= \left( \smallTransition \left(1-\smallTransition^{i - 1}\right) + (1 - \smallTransition) \right) \left(1-\smallTransition^{\numLevels}\right)^{t} \cdot C 
		&(\text{factorizing}) \\
		&= \left(1-\smallTransition^{i}\right)\left(1-\smallTransition^{\numLevels}\right)^{t} \cdot C \,.
		&(\text{adding})
	\end{align*}
	This completes the induction proof over $i$ for $t$, which in turn completes the induction proof for $t$.
	
	\qed\end{proof}

\begin{proof}[Proof of {\Cref{Result: Encompassing}}]
	Consider an MC $M$, a state $\state \in S$, and a guess $\guess$.
	Let $f \le \val_{M[\state = \guess]}$ be the lower bound on the value function of the MC where $\state$ is made target with value $\guess$.
	Let $\guess' = \Call{BellmanUpdate}{M, f, s}$.
	We prove that, for all $\eps > 0$, if $\guess' + \smallTransition^{\numStates} \eps > \guess$, then $\val_{M}(\state) > \guess - \eps$.
	
	Consider an arbitrary $\eps > 0$, the Markov Chain $M[\state  = \guess - \eps]$ and the function $f' \colon \States \to \RR^+$ defined as
	\[
	f'(\state') = \begin{cases}
		\guess - \eps
		&\state' = \state \\
		f(\state') - \eps \, \PP_{\state'}( \exists t \, \state_t = \state) 
		&\state' \not = \state
	\end{cases}
	\]
	First, we show that $f' \le \val_{M[\state = \guess - \eps]}$ for both objectives.
	Then, we argue that $\Call{BellmanUpdate}{M, f', \state} > \guess - \eps$.
	Finally, applying \Cref{Result: from_paper} to $M$, $f'$, and guess $\guess - \eps$, we conclude that $\val_{M}(\state) > \guess - \eps$.
	
	We prove that $f' \le \val_{M[\state = \guess - \eps]}$.
	Note that the inequality holds at state $\state$. 
	Therefore, consider $\state' \not = \state$. 
	We partition the plays starting at $\state'$ depending on eventually reaching $\state$ or not.
	Formally, recall that $t^* = \inf \{ t \ge 0 : s_t \in T \}$. 
	Consider both objectives separately.
	For weighted reachability objectives,
	\begin{align*}
		\val_{M[\state = \guess]}(\state'; \Reach) 
		&= \guess \, \PP_{\state'} ( \exists t \, \state_t = \state ) 
		+ \EE_{\state'} ( w(\state_{t^*}) \mathds{1}[t^* < \infty, \forall t \, \state_t \not = \state] ) \\
		&= \val_{M[\state = \guess - \eps]}(\state'; \Reach) + \eps \, \PP_{\state'} ( \exists t \, \state_t = \state ) \,.
	\end{align*}
	So for weighted reachability objectives $f' \le \val_{M[\state = \guess - \eps]}$.
	For SSP objectives, 
	\begin{align*}
		\val_{M[\state = \guess]}(\state'; \SSP) 
		&= \EE_{\state'} \left( \sum_{u = 1}^{\infty} \mathds{1}[\state_u = \state] \left( \guess + \sum_{t = 0}^{u-1} w[\state = \guess](\state_{t^*}) \right) \right) \\
		&\qquad + \EE_{\state'} \left( \mathds{1}[\forall t \, \state_t \not = \state] \sum_{t = 0}^{t^*} w[\state = \guess](\state_{t^*}) \right) \\
		&= \guess \PP_{\state'}(\exists t \, s_t = s) + \EE_{\state'} \left( \sum_{u = 1}^{\infty} \mathds{1}[\state_u = \state] \sum_{t = 0}^{u-1} w[\state = \guess](\state_{t^*}) \right) \\
		&\qquad + \EE_{\state'} \left( \mathds{1}[\forall t \, \state_t \not = \state] \sum_{t = 0}^{t^*} w[\state = \guess](\state_{t^*}) \right) \\
		&= \val_{M[\state = \guess - \eps]}(\state'; \SSP) + \eps \, \PP_{\state'} ( \exists t \, \state_t = \state ) \,.
	\end{align*}
	So for SSP objectives $f' \le \val_{M[\state = \guess - \eps]}$ too.
	
	We argue that $\Call{BellmanUpdate}{M, f', \state} > \guess - \eps$.
	Indeed, 
	\begin{align*}
		&\Call{BellmanUpdate}{M,f',s}\\
		&\qquad = \sum_{s'\in S} \prob(\state,s')f'(s')
		&\left(\text{def. Bellman update}\right)\\
		&\qquad = \sum_{s'\in S} \prob(\state,s')(f(\state) - \eps \, \PP_{\state'} ( \exists t \, \state_t = \state ) )
		&\left(\text{def. $f'$}\right)\\
		&\qquad = \sum_{s'\in S} \prob(\state,s')f(\state) - \sum_{s'\in S} \prob(\state,s') \eps \, \PP_{\state'} ( \exists t \, \state_t = \state ) 
		&(\text{factor in})\\
		&\qquad = \guess' - \sum_{s'\in S} \prob(\state,s') \eps \, \PP_{\state'} ( \exists t \, \state_t = \state ) 
		&\left(\text{def. of $\guess'$}\right)\\
		&\qquad \ge \guess' - (1-\smallTransition^{\numStates})\eps
		&\left(\Cref{observation:no_return}\right)\\
		&\qquad > \guess - \eps \,.
		&\left(\guess' + \smallTransition^{\numStates} \eps > \guess\right)
	\end{align*}
	
	Finally, applying \Cref{Result: from_paper} to $M[\state = \guess - \eps]$, $f'$ and guess $\guess - \eps$, we have $\val_{M}(\state) > \guess - \eps$. 
	Since $\eps$ is arbitrary we conclude the proof.
	
	\qed\end{proof}

\begin{proof}[Proof of {\Cref{Result: Correctness of guessing}}]
	Given an MC $M$, an approximation error $\eps$, and a set of states $I \subseteq \States$, we show that the procedure given by \Cref{Algorithm: Approximate Value of preprocessed MC}, i.e., $(l,u) = \Call{SolveWithGuessingSet}{M, \eps, I}$, satisfies that $l \le \val_M \le u$ and $\| u - l \|_\infty \le \eps$.
	
	Fix an arbitrarity state $\state$. 
	We show that $\lb \le \val_{M}(\state) \le \ub$.
	On line~\ref{line:start_bounds}, the initial bounds are set so that they contain the value.
	Every time the lower bound is updated on line~\ref{line:lower}, we have that $\lb$ is the lower bound on $M[\state = \guess]$ and
	$\Call{BellmanUpdate}{l,s}>\guess$.
	Therefore, by \Cref{Result: from_paper}, we have $\lb = \guess < \val_{M}(\state)$.
	For the upper bound, the update on line~\ref{line:upper} is symmetric.
	If no condition is invoked, then it holds that $|u - l| \le \frac{1}{4}\eps \smallTransition^{\numStates}$.
	In particular, it holds that $\Call{BellmanUpdate}{l,\state} + \frac{1}{4}\eps \smallTransition^{\numStates} > \Call{BellmanUpdate}{u,\state} > \guess$.
	Therefore, by \Cref{lemma:encompassing}, we conclude that $\val_{M}(\state) > \guess - \frac{1}{4}\eps$.
	Symmetrically, the statement holds for the upper bound.
	In the last two updates, for the lower bound $l$ on line~\ref{line:final_lower}, since $\val_{M}(\state) \ge \lb$.
	Therefore, at the end of the algorithm, we have that $l \le \val_{M[\state = \lb]}$.
	Again, the symmetric statement holds for $u$.
	
	Lastly, we show that $\| u - l \|_\infty \le \eps$.
	Before \Cref{line:final_lower}, we have that $| \ub - \lb | \le \eps / 2$.
	After \Cref{line:final_upper}, $\lb$ is at most $\frac{\eps}{4}$ apart from $\val_{M[\state = \lb]}$, and similarly for $u$ and $\val_{M[\state = \ub]}$.
	Therefore, $u$ and $l$ are at most $\eps$ apart.
	
	\qed\end{proof}

\begin{proof}[Proof of {\Cref{Result: subexponential MC Algorith}}]
	Consider an MC $M$ and an approximation error $\eps$.
	Let $I$ be the set given by \Cref{Algorithm: Mark to Guess}.
	We show that the number of calls of $\Call{BellmanUpdate}{}$ during \Cref{Algorithm: Approximate Value of preprocessed MC} is at most 
	\[
	\left( \log \left( \frac{ w_{\max} }{ \eps }\right) \frac{\numStates }{ \smallTransition  } \right)^{\calO \left( \sqrt{\numStates} \right)} \,.
	\]
	
	The proof is by induction on $|I|$. 
	Consider the base case $|I| = 0$. 
	By \Cref{Algorithm: Mark to Guess}, there are at most $\sqrt{\numStates}$ levels.
	Then, by \Cref{Result: Levels VI}, we require at most the following number of Bellman updates:
	\[
	\log \left( \frac{\sqrt{\numStates} w_{\max} }{ \eps \smallTransition } \right) \frac{ \sqrt{\numStates} }{ \smallTransition  } 
	\le \left( \log \left( \frac{ w_{\max} }{ \eps }\right) \frac{\numStates }{ \smallTransition  } \right)^{\calO \left( \sqrt{\numStates} \right)} \,.
	\]
	This proves the base case.
	For the inductive step, there is a simple recursion.
	Indeed, computing $\Call{Solve}{M, \eps, I}$ leads to at most 
	\[
	\log \left( \frac{2}{ \eps \smallTransition^{\numStates}} \right)
	\] 
	calls to the procedure $\Call{Solve}{M, \eps \smallTransition^{\numStates} / 2, I \setminus \{ \state \}}$.
	Recall that, by \Cref{Result: Preprocessing sizes}, we have that $|I| \le 9 \sqrt{\numStates}$ and the number of levels of $\Call{Guess}{M, I}$ is at most $\sqrt{\numStates}$.
	Therefore, the highest precision a Markov Chain is solved is $\eps \left( \smallTransition^{\numStates} / 2 \right)^{9 \sqrt{\numStates}}$.
	Then, the number of calls to $\Call{Solve}{M,\eps, \emptyset}$ is
	\[
	\left( \log \left( \frac{2^{\sqrt{\numStates}}}{ \eps \smallTransition^{\numStates \sqrt{\numStates}}} \right) \right)^{ 9 \sqrt{\numStates}}
	\le \left( \numStates \log \left( \frac{1}{ \eps \smallTransition} \right) \right)^{ 18 \sqrt{\numStates}} \,.
	\] 
	Lastly, by \Cref{Result: Levels VI} and the maximum precision used, each of these calls requires at most the following number of Bellman updates:
	\[
	\log \left( \frac{\sqrt{\numStates} w_{\max} }{ \eps \left( \smallTransition^{\numStates} / 2 \right)^{\sqrt{\numStates}} \smallTransition } \right) \frac{ \sqrt{\numStates} }{ \smallTransition  } 
	\le 2 \log \left( \frac{w_{\max} }{ \eps } \right) \frac{ \numStates^3 }{ \smallTransition^2 } \,.
	\]
	
	Using both bounds together, we deduce that we use in total at most 
	\[
	\left( \numStates \log \left( \frac{1}{ \eps \smallTransition} \right) \right)^{ 18 \sqrt{\numStates}}
	\cdot 2 \log \left( \frac{w_{\max} }{ \eps } \right) \frac{ \numStates^3 }{ \smallTransition^2 } 
	\]
	Bellman updates, which is at most $\left( \numStates \log (w_{\max} / \eps ) / \smallTransition \right)^{\calO \left( \sqrt{\numStates} \right)}$.

	\paragraph{Memory Requirements.}
	The depth of the recursion is only $\calO(\sqrt{\numStates})$. 
	Also, in every call of the recursion, only $\calO(\numStates)$ numbers need to be stored.
	Therefore, the overall memory requirement is $\calO(\numStates^{3/2})$.
	
	\qed\end{proof}

\begin{definition}[Maximal End Component (MEC)]
	\label{def:MEC}
	Consider an MDP $P$. Let a subgraph $G' = (\States' = \States_p' \cup \States_d', E')$, where $E' \defas \{ (\state, \state') : \state' \in \Neighbors(\state) \cap \States' \}$, of the labeled graph of $M$.
	The subgraph $G'$ is called an end-component if the following properties hold.
	\begin{itemize}
		\item \emph{Closed.} For all states $\state \in \States_p'$, we have that $\Neighbors(s) \subseteq \States'$.
		\item \emph{Strongly connected.} The graph $G'$ is strongly connected.
	\end{itemize}
	We say that the end component $G'$ is maximal if it is maximal with respect to graph inclusion.
\end{definition}

\begin{proof}[Proof of \Cref{Result: Correctness of practical guessing}]
	We prove the termination of \Cref{Algorithm: Practical guessing} by induction on the cardinality of $\States \setminus \Target$. 
	For the base case, consider $|\States \setminus \Target| = |\{ \state \}| = 1$. 
	Then, after only one Bellman update, the value of $\state$ is computed and therefore $\Call{Verify}{P, l, u, s, \guess, \eps, K_2}$ never returns Inconclusive$(l', u')$. 
	Since the length of the interval $(l_\state, u_\state)$ is divided by at least $2$ in each iteration, it terminates in at most $\log \eps$ steps.
	For the inductive case, by induction, recursive calls always terminate. 
	Therefore, the main while loop always advances.
	We prove that all possible branches of $\Call{Verify}{P, l, u, s, \guess, \eps, K_2}$ tighten the bounds guaranteeing termination.
	We proceed by cases.
	For the case Bounds$(l', u', k)$ termination is immediate.
	Assume the case Lower$(l', k)$. 
	By \Cref{Result: from_paper}, we have that $\guess$ is a lower bound. 
	Therefore, $l' \le \val$. 
	Moreover, consider $\overline{l}$ the result of applying $k$ Belmman updates to $P$ from $l$.
	By monotonicity of the Bellman update, we have that $\overline{l} \le l$. 
	Therefore, in this case, we tighten the bounds at least as much as $k$ Belmman updates would do.
	The case Upper$(u', k)$ is symmetric.
	Lastly, consider the case Inconclusive$(l', u')$. 
	If $\guess < \Call{BellmanUpdate}{l'', s}$, then the argument follows as in the case Lower$(l', k)$ and the obtained bounds tighten at least as much as $K_2$ Bellman updates.
	Similarly when $\guess > \Call{BellmanUpdate}{u'', s}$.
	For the last case, note that, if $\| u'' - l'' \|_\infty \le \eps \smallTransition^{\numStates} / ( 4 + 6 \smallTransition^{\numStates})$, then $\| u - l \|_\infty \le \eps /2$, which tightens the bounds by a factor of two. 
	In conclusion, all cases provide sufficient improvement in the bounds to guarantee termination.
	This concludes our proof by induction.
	
	Correctness is guaranteed by the invariant $l \le \val \le u$. 
	Indeed, \Cref{Result: from_paper} implies this invariant in all cases except when the guess was approximately correct, i.e., \Cref{line: guess was approximately correct}. 
	In this case, \Cref{Result: Encompassing} guarantees the invariant.
	With this, a proof by induction on $|\States \setminus \Target|$ for correctness is direct.
	
	\qed\end{proof}

\section{Fixpoint of Bellman Updates}

In general, the Bellman update operator does not have a unique fixpoint.
For example, for weighted reachability objectives, if there is a state with a loop transition, then every fixpoint can be modified at this state and remain a fixpoint.
Fortunately, there is a simple procedure called collapsing MECs that transforms an MDP into another where the Bellman update has a unique fixpoint. 
Moreover, this transformation preserves the value of the states, i.e., the values of the states in the original MDP can be deduced from the value of the states in the transformed MDP.

The collapsing procedure is as follows: 
(i)~introduce new target states with weights $w_{\max}$ and $0$;
(ii)~replace all target states, except the new ones, by states with the same incoming edges but with outgoing edges only to the new targets proportional to their weight;
(iii)~compute the set of states that have value $0$ and $w_{\max}$; 
(iv)~convert the value $0$ (resp. $w_{\max}$) class as a single sink state (the process ends
upon reaching that state) with value $0$ (resp. value $w_{\max}$); 
and (v), in the remaining MDP, collapse maximal end-components (MECs) to single states. 
Intuitively, collapsing converts each MEC to a single state with the union of exit choices.
This procedure requires subquadratic time~\cite{chatterjee2011FasterDynamicAlgorithms}.

\section{Experiments}
\label{Section: Experiments}


\paragraph{Instances.}

We consider all 636 instances contained in the Quantitative Verification Benchmark Set (QVBS)~\cite{hartmanns2019QuantitativeVerificationBenchmark}.
There are 162 instances where some of the algorithms considered timed out (at 600 seconds) or failed.
From these 162 instances, there are 153 in which all algorithms failed or timed out.
In the remaining 11 instances, each algorithm returns an answer as follows:
IVI in 3 instances; 
OVI in 6 instances; 
SVI in 7 instances; 
GVI in 3 instances.
Omitting these 162 instances leaves a total of 474 instances that we analyze.
Each instance consists of a model, parameters and a property.
All properties can be stated as either reachability or SSP objective.

Instances in QVBS are highly diverse. 
They come from established probabilistic verification and planning benchmarks, industrial case studies, models of biological systems, dynamic fault trees, and Petri net examples. 
Some of the instances have been used in previous benchmarkset, for example in the PRISM benchmark set~\cite{kwiatkowsa2012PRISMBenchmarkSuite}.

\paragraph{Experimental Setting.} 
The results were obtained on Ubuntu 20.04 with an Intel Core i7-11800H processor (4.6 GHz, 24 MB cache) using 16 GB of RAM. 
For each instance, the algorithms were run to compute the solution with at most $10^{-3}$ additive error within a time limit of 600 seconds. 
Instances for which some algorithm failed to provide a correct result within the time limit were omitted.

\paragraph{Detailed experimental results.}
Recall that we grouped the instances in the following groups.
\begin{itemize}
	\item Group 1: instances where all algorithms are fast, i.e., they take at most 100 milliseconds (170 instances).
	\item Group 2: from the rest, those where the fastest and slowest algorithms are only at most 1.10 times of each other (135 instances).
	\item Group 3: from the rest, there is a winner among the previous VI-based approaches over our approach (83 instances).
	\item Group 4: all other instances not considered before (86 instances).
\end{itemize}
We present the following tables.
\begin{itemize}
	\item \Cref{Table: Small times} contains Group 1, i.e., instances with small times.
	
	\item \Cref{Table: Indistinguishable times} contains Group 2, i.e., instances where there is no clear winner.
	
	\item \Cref{Table: Worse times} contains Group 3, i.e., among the remaining instances, there is a winner among the previous VI-based approaches over our approach.
	The improvement of each previous approach over GVI in these instances is at most 1.15 times. 
	
	\item \Cref{Table: Best Times} contains Group 4, i.e., the times for the rest of the instances. 
\end{itemize}
At the end of each table, the sum of the times over the instances is shown. 

We expand on Group 4, where there are 86 instances.
Each instance consists of a model, parameters and a property to be computed.
The following models are present in Group 4:
\begin{itemize}	\item Randomized Consensus Protocol (\emph{consensus}), first presented in~\cite{aspnes1990FastRandomizedConsensus},
	\item Coupon Collectors (\emph{coupon}), first presented in~\cite{artho_bounded_2016},
	\item Crowds Protocol (\emph{crowds}), first presented in~\cite{reiter1998CrowdsAnonymityWeb},
	\item IEEE 802.3 CSMA/CD Protocol (\emph{csma}), first presented in~\cite{kwiatkowsa2012PRISMBenchmarkSuite},
	\item Dynamic Power Management (\emph{dpm}), first presented in~\cite{qiu1999StochasticModelingPowermanaged},
	\item EchoRing (\emph{echoring}), first presented in~\cite{dombrowski_model-checking_2016},
	\item Probabilistic Contract Signing Protocol (\emph{egl}), first presented in~\cite{even1985RandomizedProtocolSigning},
	\item Embedded Control System (\emph{embedded}), first presented in~\cite{muppala1994StochasticRewardNetsReliabilityPrediction},
	\item Exploding Blocksworld (\emph{exploding-blocksworld}), first presented in~\cite{younes2005FirstProbabilisticTrack},
	\item IEEE 1394 FireWire Root Contention Protocol (\emph{firewire}), first presented in~\cite{stoelinga1999RootContentionIEEE},
	\item Fault-Tolerant Workstation Cluster (\emph{ftwc}), first presented in~\cite{haverkort2000UseModelChecking},
	\item Haddad-Monmege (\emph{haddad-monmege}), first presented in~\cite{haddad2018IntervalIterationAlgorithm},
	\item Cyclic Server Polling System (\emph{polling}), first presented in~\cite{ibe_stochastic_1990}.
\end{itemize}

\newcolumntype{R}{p{210pt}<{\raggedright\arraybackslash}}
\begin{longtable}{ |c|R||r|r|r|r| }
	\caption{Detailed experimental results measuring time in seconds for each algorithm over instance in Group 1.}
	\label{Table: Small times}\\
	\hline
	\multicolumn{2}{|c||}{Instance} & \multicolumn{4}{c|}{Time (seconds)} \\
	\hline
	\# & Model \allowbreak @~Paramenters \allowbreak @~Property & IVI & OVI & SVI & GVI \\
	\hline
	\endfirsthead
	\multicolumn{4}{c}%
	{\tablename\ \thetable\ -- \textit{Continued from previous page}} \\
	\hline
	\multicolumn{2}{|c||}{Instance} & \multicolumn{4}{c|}{Time (seconds)} \\
	\hline
	\# & Model \allowbreak @~Paramenters \allowbreak @~Property & IVI & OVI & SVI & GVI \\
	\hline
	\endhead
	\hline \multicolumn{6}{r}{\textit{Continued on next page}} \\
	\endfoot
	\hline
	\endlastfoot
	1 & beb \allowbreak @~3, 4, 3 \allowbreak @~GaveUp & 0 & 0 & 0 & 0 \\
	2 & beb \allowbreak @~3, 4, 3 \allowbreak @~LineSeized & 0 & 0 & 0 & 0 \\
	3 & brp \allowbreak @~16, 2 \allowbreak @~p1 & 0 & 0 & 0 & 0 \\
	4 & brp \allowbreak @~16, 2 \allowbreak @~p2 & 0 & 0 & 0 & 0 \\
	5 & brp \allowbreak @~16, 2 \allowbreak @~p4 & 0 & 0 & 0 & 0 \\
	6 & brp \allowbreak @~16, 3 \allowbreak @~p1 & 0 & 0 & 0 & 0 \\
	7 & brp \allowbreak @~16, 3 \allowbreak @~p2 & 0 & 0 & 0 & 0 \\
	8 & brp \allowbreak @~16, 3 \allowbreak @~p4 & 0 & 0 & 0 & 0 \\
	9 & brp \allowbreak @~16, 4 \allowbreak @~p1 & 0 & 0 & 0 & 0 \\
	10 & brp \allowbreak @~16, 4 \allowbreak @~p2 & 0 & 0 & 0 & 0 \\
	11 & brp \allowbreak @~16, 4 \allowbreak @~p4 & 0 & 0 & 0 & 0 \\
	12 & brp \allowbreak @~16, 5 \allowbreak @~p1 & 0 & 0 & 0 & 0 \\
	13 & brp \allowbreak @~16, 5 \allowbreak @~p2 & 0 & 0 & 0 & 0 \\
	14 & brp \allowbreak @~16, 5 \allowbreak @~p4 & 0 & 0 & 0 & 0 \\
	15 & brp \allowbreak @~32, 2 \allowbreak @~p1 & 0 & 0 & 0 & 0 \\
	16 & brp \allowbreak @~32, 2 \allowbreak @~p2 & 0 & 0 & 0 & 0 \\
	17 & brp \allowbreak @~32, 2 \allowbreak @~p4 & 0 & 0 & 0 & 0 \\
	18 & brp \allowbreak @~32, 3 \allowbreak @~p1 & 0 & 0 & 0 & 0 \\
	19 & brp \allowbreak @~32, 3 \allowbreak @~p2 & 0 & 0 & 0 & 0 \\
	20 & brp \allowbreak @~32, 3 \allowbreak @~p4 & 0 & 0 & 0 & 0 \\
	21 & brp \allowbreak @~32, 4 \allowbreak @~p1 & 0 & 0 & 0 & 0 \\
	22 & brp \allowbreak @~32, 4 \allowbreak @~p2 & 0 & 0 & 0 & 0 \\
	23 & brp \allowbreak @~32, 4 \allowbreak @~p4 & 0 & 0 & 0 & 0 \\
	24 & brp \allowbreak @~32, 5 \allowbreak @~p1 & 0 & 0 & 0 & 0 \\
	25 & brp \allowbreak @~32, 5 \allowbreak @~p2 & 0 & 0 & 0 & 0 \\
	26 & brp \allowbreak @~32, 5 \allowbreak @~p4 & 0 & 0 & 0 & 0 \\
	27 & brp \allowbreak @~64, 2 \allowbreak @~p1 & 0 & 0 & 0 & 0 \\
	28 & brp \allowbreak @~64, 2 \allowbreak @~p2 & 0 & 0 & 0 & 0 \\
	29 & brp \allowbreak @~64, 2 \allowbreak @~p4 & 0 & 0 & 0 & 0 \\
	30 & brp \allowbreak @~64, 3 \allowbreak @~p1 & 0 & 0 & 0 & 0 \\
	31 & brp \allowbreak @~64, 3 \allowbreak @~p2 & 0 & 0 & 0 & 0 \\
	32 & brp \allowbreak @~64, 3 \allowbreak @~p4 & 0 & 0 & 0 & 0 \\
	33 & brp \allowbreak @~64, 4 \allowbreak @~p1 & 0 & 0 & 0 & 0 \\
	34 & brp \allowbreak @~64, 4 \allowbreak @~p2 & 0 & 0 & 0 & 0 \\
	35 & brp \allowbreak @~64, 4 \allowbreak @~p4 & 0 & 0 & 0 & 0 \\
	36 & brp \allowbreak @~64, 5 \allowbreak @~p1 & 0 & 0 & 0 & 0 \\
	37 & brp \allowbreak @~64, 5 \allowbreak @~p2 & 0 & 0 & 0 & 0 \\
	38 & brp \allowbreak @~64, 5 \allowbreak @~p4 & 0 & 0 & 0 & 0 \\
	39 & cdrive \allowbreak @~2 \allowbreak @~goal & 0 & 0 & 0 & 0 \\
	40 & cdrive \allowbreak @~3 \allowbreak @~goal & 0 & 0 & 0 & 0 \\
	41 & consensus \allowbreak @~2, 16 \allowbreak @~c1 & 0 & 0 & 0 & 0 \\
	42 & consensus \allowbreak @~2, 2 \allowbreak @~c1 & 0 & 0 & 0 & 0 \\
	43 & consensus \allowbreak @~2, 2 \allowbreak @~c2 & 0 & 0 & 0 & 0 \\
	44 & consensus \allowbreak @~2, 2 \allowbreak @~disagree & 0 & 0 & 0 & 0 \\
	45 & consensus \allowbreak @~2, 4 \allowbreak @~c1 & 0 & 0 & 0 & 0 \\
	46 & consensus \allowbreak @~2, 4 \allowbreak @~c2 & 0 & 0 & 0 & 0 \\
	47 & consensus \allowbreak @~2, 4 \allowbreak @~disagree & 0 & 0 & 0 & 0 \\
	48 & consensus \allowbreak @~2, 8 \allowbreak @~c1 & 0 & 0 & 0 & 0 \\
	49 & coupon \allowbreak @~5, 2, 5 \allowbreak @~collect\_all & 0 & 0 & 0 & 0 \\
	50 & crowds \allowbreak @~3, 10 \allowbreak @~positive & 0 & 0 & 0 & 0 \\
	51 & crowds \allowbreak @~3, 5 \allowbreak @~positive & 0 & 0 & 0 & 0 \\
	52 & crowds \allowbreak @~4, 5 \allowbreak @~positive & 0 & 0 & 0 & 0 \\
	53 & csma \allowbreak @~2, 2 \allowbreak @~all\_before\_max & 0 & 0 & 0 & 0 \\
	54 & csma \allowbreak @~2, 2 \allowbreak @~all\_before\_min & 0 & 0 & 0 & 0 \\
	55 & csma \allowbreak @~2, 2 \allowbreak @~some\_before & 0 & 0 & 0 & 0 \\
	56 & elevators \allowbreak @~a, 3, 3 \allowbreak @~goal & 0 & 0 & 0 & 0 \\
	57 & elevators \allowbreak @~b, 3, 3 \allowbreak @~goal & 0 & 0 & 0 & 0 \\
	58 & erlang \allowbreak @~10, 10, 5 \allowbreak @~PminReach & 0 & 0 & 0 & 0 \\
	59 & erlang \allowbreak @~5000, 10, 5 \allowbreak @~PminReach & 0 & 0 & 0 & 0 \\
	60 & erlang \allowbreak @~5000, 100, 5 \allowbreak @~PminReach & 0 & 0 & 0 & 0 \\
	61 & erlang \allowbreak @~5000, 100, 50 \allowbreak @~PminReach & 0 & 0 & 0 & 0 \\
	62 & firewire \allowbreak @~false, 3, 200 \allowbreak @~elected & 0 & 0 & 0 & 0 \\
	63 & firewire \allowbreak @~false, 3, 400 \allowbreak @~elected & 0 & 0 & 0 & 0 \\
	64 & firewire \allowbreak @~false, 3, 600 \allowbreak @~elected & 0 & 0 & 0 & 0 \\
	65 & firewire \allowbreak @~false, 3, 800 \allowbreak @~elected & 0 & 0 & 0 & 0 \\
	66 & firewire\_abst \allowbreak @~3 \allowbreak @~elected & 0 & 0 & 0 & 0 \\
	67 & firewire\_abst \allowbreak @~36 \allowbreak @~elected & 0 & 0 & 0 & 0 \\
	68 & firewire\_dl \allowbreak @~3, 200 \allowbreak @~deadline & 0 & 0 & 0 & 0 \\
	69 & ftwc \allowbreak @~4, 5 \allowbreak @~ReachMinIsOne & 0 & 0 & 0 & 0 \\
	70 & ij \allowbreak @~10 \allowbreak @~stable & 0 & 0 & 0 & 0 \\
	71 & leader\_sync \allowbreak @~3, 2 \allowbreak @~eventually\_elected & 0 & 0 & 0 & 0 \\
	72 & leader\_sync \allowbreak @~3, 3 \allowbreak @~eventually\_elected & 0 & 0 & 0 & 0 \\
	73 & leader\_sync \allowbreak @~3, 4 \allowbreak @~eventually\_elected & 0 & 0 & 0 & 0 \\
	74 & leader\_sync \allowbreak @~4, 2 \allowbreak @~eventually\_elected & 0 & 0 & 0 & 0 \\
	75 & leader\_sync \allowbreak @~4, 3 \allowbreak @~eventually\_elected & 0 & 0 & 0 & 0 \\
	76 & leader\_sync \allowbreak @~4, 4 \allowbreak @~eventually\_elected & 0 & 0 & 0 & 0 \\
	77 & leader\_sync \allowbreak @~5, 2 \allowbreak @~eventually\_elected & 0 & 0 & 0 & 0 \\
	78 & leader\_sync \allowbreak @~5, 3 \allowbreak @~eventually\_elected & 0 & 0 & 0 & 0 \\
	79 & leader\_sync \allowbreak @~5, 4 \allowbreak @~eventually\_elected & 0 & 0 & 0 & 0 \\
	80 & philosophers-mdp \allowbreak @~3 \allowbreak @~eat & 0 & 0 & 0 & 0 \\
	81 & philosophers \allowbreak @~4, 1 \allowbreak @~MaxPrReachDeadlock & 0 & 0 & 0 & 0 \\
	82 & pnueli-zuck \allowbreak @~3 \allowbreak @~live & 0 & 0 & 0 & 0 \\
	83 & polling \allowbreak @~3, 16 \allowbreak @~s1\_before\_s2 & 0 & 0 & 0 & 0 \\
	84 & polling \allowbreak @~4, 16 \allowbreak @~s1\_before\_s2 & 0 & 0 & 0 & 0 \\
	85 & polling \allowbreak @~5, 16 \allowbreak @~s1\_before\_s2 & 0 & 0 & 0 & 0 \\
	86 & polling \allowbreak @~6, 16 \allowbreak @~s1\_before\_s2 & 0 & 0 & 0 & 0 \\
	87 & polling \allowbreak @~7, 16 \allowbreak @~s1\_before\_s2 & 0 & 0 & 0 & 0 \\
	88 & polling \allowbreak @~8, 16 \allowbreak @~s1\_before\_s2 & 0.1 & 0 & 0 & 0 \\
	89 & bitcoin-attack \allowbreak @~20, 6 \allowbreak @~T\_MWinMin & 0 & 0 & 0 & 0 \\
	90 & consensus \allowbreak @~2, 2 \allowbreak @~steps\_max & 0 & 0 & 0 & 0 \\
	91 & consensus \allowbreak @~2, 2 \allowbreak @~steps\_min & 0 & 0 & 0 & 0 \\
	92 & consensus \allowbreak @~2, 4 \allowbreak @~steps\_max & 0 & 0 & 0 & 0 \\
	93 & consensus \allowbreak @~2, 4 \allowbreak @~steps\_min & 0 & 0 & 0 & 0 \\
	94 & coupon \allowbreak @~5, 2, 5 \allowbreak @~exp\_draws & 0 & 0 & 0 & 0 \\
	95 & csma \allowbreak @~2, 2 \allowbreak @~time\_max & 0 & 0 & 0 & 0 \\
	96 & csma \allowbreak @~2, 2 \allowbreak @~time\_min & 0 & 0 & 0 & 0 \\
	97 & erlang \allowbreak @~10, 10, 5 \allowbreak @~TminReach & 0 & 0 & 0 & 0 \\
	98 & erlang \allowbreak @~5000, 10, 5 \allowbreak @~TminReach & 0 & 0 & 0 & 0 \\
	99 & erlang \allowbreak @~5000, 100, 5 \allowbreak @~TminReach & 0 & 0 & 0 & 0 \\
	100 & erlang \allowbreak @~5000, 100, 50 \allowbreak @~TminReach & 0 & 0 & 0 & 0 \\
	101 & firewire\_abst \allowbreak @~3 \allowbreak @~rounds & 0 & 0 & 0 & 0 \\
	102 & firewire\_abst \allowbreak @~3 \allowbreak @~time\_max & 0 & 0 & 0 & 0 \\
	103 & firewire\_abst \allowbreak @~3 \allowbreak @~time\_min & 0 & 0 & 0 & 0 \\
	104 & firewire\_abst \allowbreak @~36 \allowbreak @~rounds & 0 & 0 & 0 & 0 \\
	105 & firewire\_abst \allowbreak @~36 \allowbreak @~time\_max & 0 & 0 & 0 & 0 \\
	106 & firewire\_abst \allowbreak @~36 \allowbreak @~time\_min & 0 & 0 & 0 & 0 \\
	107 & herman \allowbreak @~3 \allowbreak @~steps & 0 & 0 & 0 & 0 \\
	108 & herman \allowbreak @~5 \allowbreak @~steps & 0 & 0 & 0 & 0 \\
	109 & herman \allowbreak @~7 \allowbreak @~steps & 0 & 0 & 0 & 0 \\
	110 & herman \allowbreak @~9 \allowbreak @~steps & 0 & 0 & 0 & 0 \\
	111 & jobs \allowbreak @~5, 2 \allowbreak @~avgtime & 0 & 0 & 0 & 0 \\
	112 & jobs \allowbreak @~5, 2 \allowbreak @~completiontime & 0 & 0 & 0 & 0 \\
	113 & leader\_sync \allowbreak @~3, 2 \allowbreak @~time & 0 & 0 & 0 & 0 \\
	114 & leader\_sync \allowbreak @~3, 3 \allowbreak @~time & 0 & 0 & 0 & 0 \\
	115 & leader\_sync \allowbreak @~3, 4 \allowbreak @~time & 0 & 0 & 0 & 0 \\
	116 & leader\_sync \allowbreak @~4, 2 \allowbreak @~time & 0 & 0 & 0 & 0 \\
	117 & leader\_sync \allowbreak @~4, 3 \allowbreak @~time & 0 & 0 & 0 & 0 \\
	118 & leader\_sync \allowbreak @~4, 4 \allowbreak @~time & 0 & 0 & 0 & 0 \\
	119 & leader\_sync \allowbreak @~5, 2 \allowbreak @~time & 0 & 0 & 0 & 0 \\
	120 & leader\_sync \allowbreak @~5, 3 \allowbreak @~time & 0 & 0 & 0 & 0 \\
	121 & leader\_sync \allowbreak @~5, 4 \allowbreak @~time & 0 & 0 & 0 & 0 \\
	122 & mapk\_cascade \allowbreak @~1, 30 \allowbreak @~activated\_time & 0 & 0 & 0 & 0 \\
	123 & philosophers \allowbreak @~4, 1 \allowbreak @~MinExpTimeDeadlock & 0 & 0 & 0 & 0 \\
	124 & blocksworld \allowbreak @~5 \allowbreak @~goal & 0.1 & 0.1 & 0.1 & 0.1 \\
	125 & breakdown-queues \allowbreak @~16 \allowbreak @~Max & 0.1 & 0.1 & 0.1 & 0.1 \\
	126 & breakdown-queues \allowbreak @~16 \allowbreak @~Min & 0.1 & 0.1 & 0.1 & 0.1 \\
	127 & breakdown-queues \allowbreak @~8 \allowbreak @~Max & 0.1 & 0.1 & 0.1 & 0.1 \\
	128 & breakdown-queues \allowbreak @~8 \allowbreak @~Min & 0.1 & 0.1 & 0.1 & 0.1 \\
	129 & cdrive \allowbreak @~6 \allowbreak @~goal & 0.1 & 0.1 & 0.1 & 0.1 \\
	130 & consensus \allowbreak @~2, 8 \allowbreak @~disagree & 0 & 0.1 & 0 & 0.1 \\
	131 & consensus \allowbreak @~4, 2 \allowbreak @~c1 & 0.1 & 0.1 & 0.1 & 0.1 \\
	132 & crowds \allowbreak @~3, 15 \allowbreak @~positive & 0.1 & 0.1 & 0.1 & 0.1 \\
	133 & crowds \allowbreak @~3, 20 \allowbreak @~positive & 0.1 & 0.1 & 0.1 & 0.1 \\
	134 & crowds \allowbreak @~4, 10 \allowbreak @~positive & 0.1 & 0.1 & 0.1 & 0.1 \\
	135 & crowds \allowbreak @~5, 5 \allowbreak @~positive & 0 & 0 & 0 & 0.1 \\
	136 & crowds \allowbreak @~6, 5 \allowbreak @~positive & 0.1 & 0.1 & 0.1 & 0.1 \\
	137 & csma \allowbreak @~2, 4 \allowbreak @~all\_before\_max & 0.1 & 0.1 & 0.1 & 0.1 \\
	138 & csma \allowbreak @~2, 4 \allowbreak @~all\_before\_min & 0.1 & 0.1 & 0.1 & 0.1 \\
	139 & csma \allowbreak @~2, 4 \allowbreak @~some\_before & 0.1 & 0.1 & 0.1 & 0.1 \\
	140 & dpm \allowbreak @~4, 4, 25 \allowbreak @~PmaxQueue1Full & 0.1 & 0.1 & 0.1 & 0.1 \\
	141 & dpm \allowbreak @~4, 4, 25 \allowbreak @~PmaxQueuesFull & 0.1 & 0.1 & 0.1 & 0.1 \\
	142 & dpm \allowbreak @~4, 4, 25 \allowbreak @~PminQueue1Full & 0.1 & 0.1 & 0.1 & 0.1 \\
	143 & dpm \allowbreak @~4, 4, 5 \allowbreak @~PmaxQueue1Full & 0.1 & 0.1 & 0.1 & 0.1 \\
	144 & dpm \allowbreak @~4, 4, 5 \allowbreak @~PmaxQueuesFull & 0.1 & 0.1 & 0.1 & 0.1 \\
	145 & dpm \allowbreak @~4, 4, 5 \allowbreak @~PminQueue1Full & 0.1 & 0.1 & 0.1 & 0.1 \\
	146 & elevators \allowbreak @~a, 11, 9 \allowbreak @~goal & 0.1 & 0.1 & 0.1 & 0.1 \\
	147 & firewire\_dl \allowbreak @~3, 400 \allowbreak @~deadline & 0.1 & 0.1 & 0.1 & 0.1 \\
	148 & firewire\_dl \allowbreak @~36, 200 \allowbreak @~deadline & 0.1 & 0.1 & 0.1 & 0.1 \\
	149 & ftwc \allowbreak @~8, 5 \allowbreak @~ReachMinIsOne & 0.1 & 0.1 & 0.1 & 0.1 \\
	150 & nand \allowbreak @~20, 1 \allowbreak @~reliable & 0.1 & 0.1 & 0.1 & 0.1 \\
	151 & pacman \allowbreak @~5 \allowbreak @~crash & 0.1 & 0.1 & 0.1 & 0.1 \\
	152 & consensus \allowbreak @~2, 8 \allowbreak @~steps\_max & 0.1 & 0.1 & 0.1 & 0.1 \\
	153 & consensus \allowbreak @~2, 8 \allowbreak @~steps\_min & 0.1 & 0.1 & 0.1 & 0.1 \\
	154 & csma \allowbreak @~2, 4 \allowbreak @~time\_max & 0.1 & 0.1 & 0.1 & 0.1 \\
	155 & csma \allowbreak @~2, 4 \allowbreak @~time\_min & 0.1 & 0.1 & 0.1 & 0.1 \\
	156 & eajs \allowbreak @~2, 100, 5 \allowbreak @~ExpUtil & 0.1 & 0.1 & 0.1 & 0.1 \\
	157 & firewire \allowbreak @~false, 3, 200 \allowbreak @~time\_max & 0.1 & 0.1 & 0.1 & 0.1 \\
	158 & firewire \allowbreak @~false, 3, 200 \allowbreak @~time\_min & 0 & 0.1 & 0 & 0.1 \\
	159 & firewire \allowbreak @~false, 3, 200 \allowbreak @~time\_sending & 0.1 & 0.1 & 0.1 & 0.1 \\
	160 & firewire \allowbreak @~false, 3, 400 \allowbreak @~time\_max & 0.1 & 0.1 & 0.1 & 0.1 \\
	161 & firewire \allowbreak @~false, 3, 400 \allowbreak @~time\_min & 0 & 0 & 0 & 0.1 \\
	162 & firewire \allowbreak @~false, 3, 400 \allowbreak @~time\_sending & 0.1 & 0.1 & 0.1 & 0.1 \\
	163 & firewire \allowbreak @~false, 3, 600 \allowbreak @~time\_max & 0.1 & 0.1 & 0.1 & 0.1 \\
	164 & firewire \allowbreak @~false, 3, 600 \allowbreak @~time\_min & 0 & 0 & 0 & 0.1 \\
	165 & firewire \allowbreak @~false, 3, 600 \allowbreak @~time\_sending & 0.1 & 0.1 & 0.1 & 0.1 \\
	166 & firewire \allowbreak @~false, 3, 800 \allowbreak @~time\_max & 0.1 & 0.1 & 0.1 & 0.1 \\
	167 & firewire \allowbreak @~false, 3, 800 \allowbreak @~time\_min & 0.1 & 0 & 0 & 0.1 \\
	168 & firewire \allowbreak @~false, 3, 800 \allowbreak @~time\_sending & 0.1 & 0.1 & 0.1 & 0.1 \\
	169 & flexible-manufacturing \allowbreak @~3, 1 \allowbreak @~M3Fail\_E & 0 & 0 & 0 & 0.1 \\
	170 & herman \allowbreak @~11 \allowbreak @~steps & 0.1 & 0.1 & 0.1 & 0.1 \\
	\hline
	\hline
	& SUM & 4.2 & 4.2 & 4.0 & 4.7
\end{longtable}

\newcolumntype{R}{p{210pt}<{\raggedright\arraybackslash}}
\begin{longtable}{ |c|R||r|r|r|r| }
	\caption{Detailed experimental results measuring time in seconds for each algorithm over instance in Group 2.}
	\label{Table: Indistinguishable times}\\
	\hline
	\multicolumn{2}{|c||}{Instance} & \multicolumn{4}{c|}{Time (seconds)} \\
	\hline
	\# & Model \allowbreak @~Paramenters \allowbreak @~Property & IVI & OVI & SVI & GVI \\
	\hline
	\endfirsthead
	\multicolumn{4}{c}%
	{\tablename\ \thetable\ -- \textit{Continued from previous page}} \\
	\hline
	\multicolumn{2}{|c||}{Instance} & \multicolumn{4}{c|}{Time (seconds)} \\
	\hline
	\# & Model \allowbreak @~Paramenters \allowbreak @~Property & IVI & OVI & SVI & GVI \\
	\hline
	\endhead
	\hline \multicolumn{6}{r}{\textit{Continued on next page}} \\
	\endfoot
	\hline
	\endlastfoot
	1 & breakdown-queues \allowbreak @~32 \allowbreak @~Max & 0.2 & 0.2 & 0.2 & 0.2 \\
	2 & breakdown-queues \allowbreak @~32 \allowbreak @~Min & 0.2 & 0.2 & 0.2 & 0.2 \\
	3 & csma \allowbreak @~2, 6 \allowbreak @~some\_before & 0.2 & 0.2 & 0.2 & 0.2 \\
	4 & csma \allowbreak @~3, 2 \allowbreak @~all\_before\_max & 0.2 & 0.2 & 0.2 & 0.2 \\
	5 & csma \allowbreak @~3, 2 \allowbreak @~all\_before\_min & 0.2 & 0.2 & 0.2 & 0.2 \\
	6 & dpm \allowbreak @~4, 6, 100 \allowbreak @~PmaxQueue1Full & 0.2 & 0.2 & 0.2 & 0.2 \\
	7 & dpm \allowbreak @~4, 6, 25 \allowbreak @~PmaxQueue1Full & 0.2 & 0.2 & 0.2 & 0.2 \\
	8 & dpm \allowbreak @~4, 6, 50 \allowbreak @~PmaxQueue1Full & 0.2 & 0.2 & 0.2 & 0.2 \\
	9 & egl \allowbreak @~5, 2 \allowbreak @~unfairA & 0.2 & 0.2 & 0.2 & 0.2 \\
	10 & egl \allowbreak @~5, 2 \allowbreak @~unfairB & 0.2 & 0.2 & 0.2 & 0.2 \\
	11 & csma \allowbreak @~3, 2 \allowbreak @~time\_min & 0.2 & 0.2 & 0.2 & 0.2 \\
	12 & egl \allowbreak @~5, 2 \allowbreak @~messagesA & 0.2 & 0.2 & 0.2 & 0.2 \\
	13 & egl \allowbreak @~5, 2 \allowbreak @~messagesB & 0.2 & 0.2 & 0.2 & 0.2 \\
	14 & oscillators \allowbreak @~6, 6, 0.1, 1, 0.1, 1.0 \allowbreak @~power\_consumption & 0.2 & 0.2 & 0.2 & 0.2 \\
	15 & oscillators \allowbreak @~6, 6, 0.1, 1, 0.1, 1.0 \allowbreak @~time\_to\_synch & 0.2 & 0.2 & 0.2 & 0.2 \\
	16 & cdrive \allowbreak @~10 \allowbreak @~goal & 0.3 & 0.3 & 0.3 & 0.3 \\
	17 & crowds \allowbreak @~4, 15 \allowbreak @~positive & 0.3 & 0.3 & 0.3 & 0.3 \\
	18 & csma \allowbreak @~2, 6 \allowbreak @~all\_before\_max & 0.3 & 0.3 & 0.3 & 0.3 \\
	19 & dpm \allowbreak @~4, 6, 100 \allowbreak @~PmaxQueuesFull & 0.3 & 0.3 & 0.3 & 0.3 \\
	20 & dpm \allowbreak @~4, 6, 25 \allowbreak @~PmaxQueuesFull & 0.3 & 0.3 & 0.3 & 0.3 \\
	21 & dpm \allowbreak @~4, 6, 50 \allowbreak @~PmaxQueuesFull & 0.3 & 0.3 & 0.3 & 0.3 \\
	22 & firewire\_dl \allowbreak @~3, 600 \allowbreak @~deadline & 0.3 & 0.3 & 0.3 & 0.3 \\
	23 & nand \allowbreak @~20, 2 \allowbreak @~reliable & 0.3 & 0.3 & 0.3 & 0.3 \\
	24 & philosophers \allowbreak @~12, 1 \allowbreak @~MaxPrReachDeadlock & 0.3 & 0.3 & 0.3 & 0.3 \\
	25 & csma \allowbreak @~2, 6 \allowbreak @~time\_min & 0.3 & 0.3 & 0.3 & 0.3 \\
	26 & jobs \allowbreak @~10, 3 \allowbreak @~avgtime & 0.3 & 0.3 & 0.3 & 0.3 \\
	27 & jobs \allowbreak @~10, 3 \allowbreak @~completiontime & 0.3 & 0.3 & 0.3 & 0.3 \\
	28 & breakdown-queues \allowbreak @~64 \allowbreak @~Max & 0.4 & 0.4 & 0.4 & 0.4 \\
	29 & breakdown-queues \allowbreak @~64 \allowbreak @~Min & 0.4 & 0.4 & 0.4 & 0.4 \\
	30 & egl \allowbreak @~5, 4 \allowbreak @~unfairA & 0.4 & 0.4 & 0.4 & 0.4 \\
	31 & egl \allowbreak @~5, 4 \allowbreak @~unfairB & 0.4 & 0.4 & 0.4 & 0.4 \\
	32 & firewire\_dl \allowbreak @~36, 400 \allowbreak @~deadline & 0.4 & 0.4 & 0.4 & 0.4 \\
	33 & nand \allowbreak @~20, 3 \allowbreak @~reliable & 0.4 & 0.4 & 0.4 & 0.4 \\
	34 & egl \allowbreak @~5, 4 \allowbreak @~messagesA & 0.4 & 0.4 & 0.4 & 0.4 \\
	35 & egl \allowbreak @~5, 4 \allowbreak @~messagesB & 0.4 & 0.4 & 0.4 & 0.4 \\
	36 & echoring \allowbreak @~2 \allowbreak @~MaxOffline1 & 0.5 & 0.5 & 0.5 & 0.5 \\
	37 & echoring \allowbreak @~2 \allowbreak @~MaxOffline2 & 0.5 & 0.5 & 0.5 & 0.5 \\
	38 & echoring \allowbreak @~2 \allowbreak @~MaxOffline3 & 0.5 & 0.5 & 0.5 & 0.5 \\
	39 & echoring \allowbreak @~2 \allowbreak @~MinOffline1 & 0.5 & 0.5 & 0.5 & 0.5 \\
	40 & echoring \allowbreak @~2 \allowbreak @~MinOffline2 & 0.5 & 0.5 & 0.5 & 0.5 \\
	41 & echoring \allowbreak @~2 \allowbreak @~MinOffline3 & 0.5 & 0.5 & 0.5 & 0.5 \\
	42 & nand \allowbreak @~20, 4 \allowbreak @~reliable & 0.5 & 0.5 & 0.5 & 0.5 \\
	43 & firewire \allowbreak @~true, 3, 200 \allowbreak @~time\_min & 0.5 & 0.5 & 0.5 & 0.5 \\
	44 & dpm \allowbreak @~4, 8, 100 \allowbreak @~PmaxQueue1Full & 0.6 & 0.6 & 0.6 & 0.6 \\
	45 & dpm \allowbreak @~4, 8, 25 \allowbreak @~PmaxQueue1Full & 0.6 & 0.6 & 0.6 & 0.6 \\
	46 & dpm \allowbreak @~4, 8, 5 \allowbreak @~PmaxQueue1Full & 0.6 & 0.6 & 0.6 & 0.6 \\
	47 & egl \allowbreak @~5, 6 \allowbreak @~unfairA & 0.6 & 0.6 & 0.6 & 0.6 \\
	48 & egl \allowbreak @~5, 6 \allowbreak @~unfairB & 0.6 & 0.6 & 0.6 & 0.6 \\
	49 & egl \allowbreak @~5, 6 \allowbreak @~messagesA & 0.6 & 0.6 & 0.6 & 0.6 \\
	50 & egl \allowbreak @~5, 6 \allowbreak @~messagesB & 0.6 & 0.6 & 0.6 & 0.6 \\
	51 & dpm \allowbreak @~4, 8, 100 \allowbreak @~PmaxQueuesFull & 0.7 & 0.7 & 0.7 & 0.7 \\
	52 & dpm \allowbreak @~4, 8, 25 \allowbreak @~PmaxQueuesFull & 0.7 & 0.7 & 0.7 & 0.7 \\
	53 & dpm \allowbreak @~4, 8, 5 \allowbreak @~PmaxQueuesFull & 0.7 & 0.7 & 0.7 & 0.7 \\
	54 & firewire\_dl \allowbreak @~36, 600 \allowbreak @~deadline & 0.7 & 0.7 & 0.7 & 0.7 \\
	55 & egl \allowbreak @~5, 8 \allowbreak @~unfairA & 0.8 & 0.8 & 0.8 & 0.8 \\
	56 & egl \allowbreak @~5, 8 \allowbreak @~unfairB & 0.8 & 0.8 & 0.8 & 0.8 \\
	57 & firewire\_dl \allowbreak @~36, 800 \allowbreak @~deadline & 1 & 1 & 1 & 1 \\
	58 & herman \allowbreak @~13 \allowbreak @~steps & 1.1 & 1 & 1.1 & 1 \\
	59 & firewire \allowbreak @~false, 36, 400 \allowbreak @~elected & 1.3 & 1.3 & 1.3 & 1.3 \\
	60 & firewire \allowbreak @~false, 36, 600 \allowbreak @~elected & 1.3 & 1.3 & 1.3 & 1.3 \\
	61 & firewire \allowbreak @~false, 36, 800 \allowbreak @~elected & 1.3 & 1.3 & 1.3 & 1.3 \\
	62 & firewire \allowbreak @~false, 36, 200 \allowbreak @~elected & 1.3 & 1.3 & 1.3 & 1.4 \\
	63 & firewire \allowbreak @~false, 36, 600 \allowbreak @~time\_min & 1.6 & 1.5 & 1.6 & 1.6 \\
	64 & firewire \allowbreak @~false, 36, 800 \allowbreak @~time\_min & 1.6 & 1.5 & 1.6 & 1.6 \\
	65 & crowds \allowbreak @~5, 15 \allowbreak @~positive & 1.8 & 1.7 & 1.7 & 1.7 \\
	66 & nand \allowbreak @~40, 1 \allowbreak @~reliable & 1.8 & 1.8 & 1.8 & 1.8 \\
	67 & pnueli-zuck \allowbreak @~5 \allowbreak @~live & 2 & 2 & 2 & 2 \\
	68 & firewire \allowbreak @~true, 3, 400 \allowbreak @~elected & 2.3 & 2.3 & 2.3 & 2.3 \\
	69 & dpm \allowbreak @~6, 4, 5 \allowbreak @~PmaxQueue1Full & 2.4 & 2.5 & 2.4 & 2.4 \\
	70 & csma \allowbreak @~4, 2 \allowbreak @~some\_before & 2.8 & 2.8 & 2.8 & 2.8 \\
	71 & dpm \allowbreak @~6, 4, 5 \allowbreak @~PmaxQueuesFull & 3.1 & 3.1 & 3.1 & 3.1 \\
	72 & elevators \allowbreak @~b, 11, 9 \allowbreak @~goal & 3.5 & 3.5 & 3.5 & 3.6 \\
	73 & nand \allowbreak @~40, 2 \allowbreak @~reliable & 3.8 & 3.8 & 3.8 & 3.8 \\
	74 & csma \allowbreak @~4, 2 \allowbreak @~all\_before\_min & 3.8 & 3.8 & 3.8 & 3.9 \\
	75 & oscillators \allowbreak @~6, 10, 0.1, 1, 0.1, 1.0 \allowbreak @~time\_to\_synch & 4 & 4 & 4 & 3.9 \\
	76 & csma \allowbreak @~4, 2 \allowbreak @~all\_before\_max & 4.1 & 4 & 3.9 & 4.1 \\
	77 & oscillators \allowbreak @~6, 10, 0.1, 1, 0.1, 1.0 \allowbreak @~power\_consumption & 4 & 4 & 4 & 4.1 \\
	78 & consensus \allowbreak @~6, 2 \allowbreak @~c1 & 5.2 & 5.4 & 5.2 & 5.3 \\
	79 & nand \allowbreak @~40, 3 \allowbreak @~reliable & 5.6 & 5.7 & 5.7 & 5.6 \\
	80 & csma \allowbreak @~3, 4 \allowbreak @~some\_before & 5.8 & 5.8 & 5.7 & 5.7 \\
	81 & firewire \allowbreak @~true, 3, 600 \allowbreak @~elected & 5.7 & 5.7 & 5.7 & 5.7 \\
	82 & csma \allowbreak @~3, 4 \allowbreak @~all\_before\_min & 6.3 & 6.3 & 6.3 & 6.2 \\
	83 & oscillators \allowbreak @~8, 8, 0.1, 1, 0.1, 1.0 \allowbreak @~power\_consumption & 6.4 & 6.5 & 6.4 & 6.5 \\
	84 & oscillators \allowbreak @~8, 8, 0.1, 1, 0.1, 1.0 \allowbreak @~time\_to\_synch & 6.5 & 6.5 & 6.5 & 6.5 \\
	85 & crowds \allowbreak @~5, 20 \allowbreak @~positive & 7 & 6.8 & 6.8 & 6.7 \\
	86 & csma \allowbreak @~3, 4 \allowbreak @~all\_before\_max & 6.7 & 6.6 & 6.7 & 6.7 \\
	87 & nand \allowbreak @~40, 4 \allowbreak @~reliable & 7.8 & 7.8 & 7.7 & 7.8 \\
	88 & nand \allowbreak @~60, 1 \allowbreak @~reliable & 8.7 & 8.8 & 8.7 & 8.7 \\
	89 & echoring \allowbreak @~50 \allowbreak @~MinFailed & 9.7 & 9.7 & 9.6 & 9.7 \\
	90 & herman \allowbreak @~15 \allowbreak @~steps & 10.1 & 9.5 & 10 & 9.7 \\
	91 & echoring \allowbreak @~50 \allowbreak @~MinOffline1 & 10.1 & 10.2 & 10.2 & 10.1 \\
	92 & echoring \allowbreak @~50 \allowbreak @~MinOffline2 & 10 & 10 & 10 & 10.1 \\
	93 & firewire \allowbreak @~true, 3, 800 \allowbreak @~elected & 10.2 & 10 & 10 & 10.1 \\
	94 & echoring \allowbreak @~50 \allowbreak @~MinOffline3 & 10.2 & 10.2 & 10.2 & 10.2 \\
	95 & echoring \allowbreak @~50 \allowbreak @~MaxOffline2 & 10.2 & 10.3 & 10.3 & 10.3 \\
	96 & crowds \allowbreak @~6, 15 \allowbreak @~positive & 10.7 & 10.5 & 10.3 & 10.4 \\
	97 & echoring \allowbreak @~50 \allowbreak @~MaxOffline1 & 10.5 & 10.4 & 10.4 & 10.4 \\
	98 & echoring \allowbreak @~50 \allowbreak @~MaxOffline3 & 10.5 & 10.5 & 10.5 & 10.5 \\
	99 & oscillators \allowbreak @~7, 10, 0.1, 1, 0.1, 1.0 \allowbreak @~time\_to\_synch & 13.2 & 13.3 & 13.1 & 13.4 \\
	100 & oscillators \allowbreak @~7, 10, 0.1, 1, 0.1, 1.0 \allowbreak @~power\_consumption & 13.1 & 13.2 & 13.1 & 13.5 \\
	101 & philosophers \allowbreak @~16, 1 \allowbreak @~MaxPrReachDeadlock & 14.3 & 14.2 & 14.1 & 14.5 \\
	102 & nand \allowbreak @~60, 2 \allowbreak @~reliable & 18.4 & 18.4 & 18.2 & 18.4 \\
	103 & echoring \allowbreak @~100 \allowbreak @~MinFailed & 19.4 & 19.5 & 19.5 & 19.4 \\
	104 & echoring \allowbreak @~100 \allowbreak @~MinOffline2 & 20.1 & 20 & 20.1 & 19.9 \\
	105 & echoring \allowbreak @~100 \allowbreak @~MinOffline1 & 20.4 & 20.3 & 20.3 & 20.2 \\
	106 & echoring \allowbreak @~100 \allowbreak @~MinOffline3 & 21.5 & 20.5 & 20.5 & 20.5 \\
	107 & echoring \allowbreak @~100 \allowbreak @~MaxOffline2 & 20.7 & 20.8 & 20.6 & 20.6 \\
	108 & dpm \allowbreak @~6, 6, 5 \allowbreak @~PmaxQueue1Full & 20.7 & 20.7 & 20.6 & 20.7 \\
	109 & echoring \allowbreak @~100 \allowbreak @~MaxOffline1 & 20.8 & 20.9 & 20.9 & 21 \\
	110 & echoring \allowbreak @~100 \allowbreak @~MaxOffline3 & 21.1 & 21.1 & 21.1 & 21 \\
	111 & dpm \allowbreak @~6, 6, 5 \allowbreak @~PmaxQueuesFull & 25 & 24.4 & 24.9 & 24.6 \\
	112 & nand \allowbreak @~60, 3 \allowbreak @~reliable & 28.4 & 28.6 & 28.6 & 28.5 \\
	113 & nand \allowbreak @~60, 4 \allowbreak @~reliable & 38.1 & 38.4 & 37.9 & 38 \\
	114 & firewire \allowbreak @~true, 36, 200 \allowbreak @~elected & 47.9 & 48 & 47.6 & 47.7 \\
	115 & oscillators \allowbreak @~8, 10, 0.1, 1, 0.1, 1.0 \allowbreak @~time\_to\_synch & 47.4 & 47.2 & 47.3 & 48.1 \\
	116 & oscillators \allowbreak @~8, 10, 0.1, 1, 0.1, 1.0 \allowbreak @~power\_consumption & 47.8 & 47.6 & 47.4 & 48.6 \\
	117 & coupon \allowbreak @~9, 4, 5 \allowbreak @~collect\_all & 46.2 & 45.2 & 45.3 & 49.3 \\
	118 & beb \allowbreak @~4, 8, 7 \allowbreak @~GaveUp & 53.6 & 53.6 & 53.9 & 54 \\
	119 & crowds \allowbreak @~6, 20 \allowbreak @~positive & 69.2 & 67.2 & 66 & 67.5 \\
	120 & beb \allowbreak @~4, 8, 7 \allowbreak @~LineSeized & 70.2 & 70 & 70 & 70.1 \\
	121 & jobs \allowbreak @~15, 3 \allowbreak @~completiontime & 81.9 & 79.2 & 81.7 & 80 \\
	122 & jobs \allowbreak @~15, 3 \allowbreak @~avgtime & 79.6 & 77.6 & 78 & 81 \\
	123 & pacman \allowbreak @~60 \allowbreak @~crash & 82.7 & 82.7 & 83.4 & 83 \\
	124 & herman \allowbreak @~17 \allowbreak @~steps & 95.2 & 91.2 & 94.9 & 91.4 \\
	125 & dpm \allowbreak @~6, 8, 5 \allowbreak @~PmaxQueue1Full & 100.7 & 99.8 & 100.2 & 99.6 \\
	126 & dpm \allowbreak @~8, 4, 5 \allowbreak @~PmaxQueue1Full & 102.1 & 102.1 & 103.6 & 102.3 \\
	127 & dpm \allowbreak @~6, 8, 5 \allowbreak @~PmaxQueuesFull & 117.3 & 123.9 & 122.9 & 115.5 \\
	128 & dpm \allowbreak @~8, 4, 5 \allowbreak @~PmaxQueuesFull & 133.2 & 132.6 & 134.4 & 132.3 \\
	129 & pacman \allowbreak @~100 \allowbreak @~crash & 189.8 & 189.2 & 204.6 & 189.5 \\
	130 & firewire \allowbreak @~true, 36, 400 \allowbreak @~elected & 311 & 310.1 & 311.7 & 309.6 \\
	131 & egl \allowbreak @~10, 2 \allowbreak @~unfairA & 329.7 & 330.3 & 328.9 & 329.9 \\
	132 & egl \allowbreak @~10, 2 \allowbreak @~unfairB & 330.5 & 331.3 & 340.9 & 330.1 \\
	133 & csma \allowbreak @~3, 6 \allowbreak @~some\_before & 373.2 & 372.7 & 375.2 & 374.2 \\
	134 & csma \allowbreak @~3, 6 \allowbreak @~all\_before\_min & 396.8 & 391.4 & 397.3 & 395.5 \\
	135 & csma \allowbreak @~3, 6 \allowbreak @~all\_before\_max & 477.2 & 476.9 & 473.5 & 475.4 \\
	\hline
	\hline
	& SUM & 4144.6 & 4129.2 & 4167.1 & 4134.1
\end{longtable}

\newcolumntype{R}{p{210pt}<{\raggedright\arraybackslash}}
\begin{longtable}{ |c|R||r|r|r|r| }
	\caption{Detailed experimental results measuring time in seconds for each algorithm over instance in Group 3.}
	\label{Table: Worse times}\\
	\hline
	\multicolumn{2}{|c||}{Instance} & \multicolumn{4}{c|}{Time (seconds)} \\
	\hline
	\# & Model \allowbreak @~Paramenters \allowbreak @~Property & IVI & OVI & SVI & GVI \\
	\hline
	\endfirsthead
	\multicolumn{6}{c}%
	{\tablename\ \thetable\ -- \textit{Continued from previous page}} \\
	\hline
	\multicolumn{2}{|c||}{Instance} & \multicolumn{4}{c|}{Time (seconds)} \\
	\hline
	\# & Model \allowbreak @~Paramenters \allowbreak @~Property & IVI & OVI & SVI & GVI \\
	\hline
	\endhead
	\hline \multicolumn{6}{r}{\textit{Continued on next page}} \\
	\endfoot
	\hline
	\endlastfoot
	1 & consensus \allowbreak @~2, 8 \allowbreak @~c2 & 0 & 0.2 & 0 & 0.1 \\
	2 & consensus \allowbreak @~4, 4 \allowbreak @~c1 & 0.1 & 0.1 & 0.1 & 0.2 \\
	3 & polling \allowbreak @~10, 16 \allowbreak @~s1\_before\_s2 & 0.2 & 0.2 & 0.1 & 0.2 \\
	4 & mapk\_cascade \allowbreak @~2, 30 \allowbreak @~activated\_time & 0.2 & 0.1 & 0.1 & 0.2 \\
	5 & consensus \allowbreak @~2, 16 \allowbreak @~disagree & 0.2 & 2.6 & 0.2 & 0.3 \\
	6 & csma \allowbreak @~2, 6 \allowbreak @~all\_before\_min & 0.2 & 0.2 & 0.2 & 0.3 \\
	7 & csma \allowbreak @~3, 2 \allowbreak @~time\_max & 0.3 & 0.2 & 0.3 & 0.3 \\
	8 & crowds \allowbreak @~5, 10 \allowbreak @~positive & 0.3 & 0.3 & 0.3 & 0.4 \\
	9 & csma \allowbreak @~2, 6 \allowbreak @~time\_max & 0.4 & 0.3 & 0.4 & 0.4 \\
	10 & firewire\_dl \allowbreak @~3, 800 \allowbreak @~deadline & 0.4 & 0.4 & 0.4 & 0.5 \\
	11 & consensus \allowbreak @~2, 16 \allowbreak @~steps\_min & 0.3 & 0.4 & 0.2 & 0.5 \\
	12 & philosophers \allowbreak @~12, 1 \allowbreak @~MinExpTimeDeadlock & 0.5 & 0.6 & 0.4 & 0.5 \\
	13 & polling \allowbreak @~12, 16 \allowbreak @~s1\_before\_s2 & 0.9 & 0.8 & 0.6 & 0.7 \\
	14 & consensus \allowbreak @~4, 2 \allowbreak @~steps\_max & 0.6 & 0.5 & 0.4 & 0.7 \\
	15 & dpm \allowbreak @~4, 6, 100 \allowbreak @~TminQueuesFull & 0.6 & 4.3 & 0.5 & 0.7 \\
	16 & dpm \allowbreak @~4, 6, 25 \allowbreak @~TminQueuesFull & 0.6 & 4.4 & 0.5 & 0.7 \\
	17 & flexible-manufacturing \allowbreak @~9, 1 \allowbreak @~M3Fail\_E & 0.4 & 1 & 0.3 & 0.7 \\
	18 & consensus \allowbreak @~2, 16 \allowbreak @~steps\_max & 0.4 & 0.4 & 0.2 & 0.8 \\
	19 & coupon \allowbreak @~7, 3, 5 \allowbreak @~exp\_draws & 0.8 & 0.6 & 0.7 & 0.8 \\
	20 & dpm \allowbreak @~4, 6, 50 \allowbreak @~TminQueuesFull & 0.6 & 4.4 & 0.5 & 0.8 \\
	21 & eajs \allowbreak @~3, 150, 7 \allowbreak @~ExpUtil & 0.8 & 0.6 & 0.8 & 0.8 \\
	22 & egl \allowbreak @~5, 8 \allowbreak @~messagesB & 0.8 & 0.8 & 0.8 & 0.9 \\
	23 & oscillators \allowbreak @~6, 8, 0.1, 1, 0.1, 1.0 \allowbreak @~power\_consumption & 0.8 & 0.8 & 0.8 & 0.9 \\
	24 & oscillators \allowbreak @~6, 8, 0.1, 1, 0.1, 1.0 \allowbreak @~time\_to\_synch & 0.8 & 0.8 & 0.8 & 0.9 \\
	25 & firewire \allowbreak @~true, 3, 200 \allowbreak @~time\_max & 1 & 0.5 & 1 & 1 \\
	26 & firewire \allowbreak @~true, 3, 200 \allowbreak @~time\_sending & 1 & 0.5 & 1 & 1 \\
	27 & crowds \allowbreak @~6, 10 \allowbreak @~positive & 1.1 & 1 & 1 & 1.3 \\
	28 & polling \allowbreak @~13, 16 \allowbreak @~s1\_before\_s2 & 2 & 1.9 & 1.3 & 1.7 \\
	29 & firewire \allowbreak @~false, 36, 200 \allowbreak @~time\_min & 1.6 & 1.5 & 1.6 & 1.7 \\
	30 & firewire \allowbreak @~false, 36, 400 \allowbreak @~time\_min & 1.6 & 1.5 & 1.6 & 1.7 \\
	31 & mapk\_cascade \allowbreak @~3, 30 \allowbreak @~activated\_time & 2.5 & 0.6 & 0.5 & 2 \\
	32 & consensus \allowbreak @~4, 4 \allowbreak @~steps\_min & 2 & 2.8 & 1.3 & 2.1 \\
	33 & dpm \allowbreak @~4, 8, 100 \allowbreak @~TminQueuesFull & 1.8 & 15.6 & 1.5 & 2.3 \\
	34 & dpm \allowbreak @~4, 8, 25 \allowbreak @~TminQueuesFull & 1.8 & 15.9 & 1.5 & 2.3 \\
	35 & dpm \allowbreak @~4, 8, 5 \allowbreak @~TminQueuesFull & 1.8 & 15.9 & 1.5 & 2.3 \\
	36 & firewire \allowbreak @~true, 3, 400 \allowbreak @~time\_min & 2.7 & 2.4 & 2.7 & 2.8 \\
	37 & polling \allowbreak @~14, 16 \allowbreak @~s1\_before\_s2 & 4.5 & 4.4 & 2.9 & 3.7 \\
	38 & firewire \allowbreak @~false, 36, 400 \allowbreak @~time\_sending & 3.9 & 1.6 & 3.8 & 3.9 \\
	39 & firewire \allowbreak @~false, 36, 200 \allowbreak @~time\_sending & 3.9 & 1.6 & 3.8 & 4 \\
	40 & firewire \allowbreak @~false, 36, 600 \allowbreak @~time\_sending & 3.9 & 1.6 & 3.8 & 4 \\
	41 & csma \allowbreak @~4, 2 \allowbreak @~time\_min & 4.1 & 3.7 & 4.1 & 4.1 \\
	42 & firewire \allowbreak @~false, 36, 600 \allowbreak @~time\_max & 3.9 & 2.3 & 3.8 & 4.1 \\
	43 & firewire \allowbreak @~false, 36, 800 \allowbreak @~time\_max & 3.9 & 2.3 & 3.9 & 4.1 \\
	44 & firewire \allowbreak @~false, 36, 800 \allowbreak @~time\_sending & 4 & 1.7 & 3.9 & 4.1 \\
	45 & firewire \allowbreak @~false, 36, 200 \allowbreak @~time\_max & 4 & 2.3 & 3.9 & 4.2 \\
	46 & firewire \allowbreak @~false, 36, 400 \allowbreak @~time\_max & 4 & 2.2 & 3.9 & 4.2 \\
	47 & consensus \allowbreak @~4, 4 \allowbreak @~steps\_max & 3.2 & 3.5 & 1.7 & 4.9 \\
	48 & firewire \allowbreak @~true, 3, 400 \allowbreak @~time\_sending & 5.1 & 2.6 & 5.1 & 5.1 \\
	49 & firewire \allowbreak @~true, 3, 400 \allowbreak @~time\_max & 5.1 & 2.7 & 5.1 & 5.2 \\
	50 & csma \allowbreak @~4, 2 \allowbreak @~time\_max & 6.8 & 3.6 & 6.7 & 6.5 \\
	51 & csma \allowbreak @~3, 4 \allowbreak @~time\_min & 7.2 & 6.2 & 7.3 & 7.2 \\
	52 & firewire \allowbreak @~true, 3, 600 \allowbreak @~time\_min & 7 & 6.2 & 7 & 7.4 \\
	53 & eajs \allowbreak @~4, 200, 9 \allowbreak @~ExpUtil & 7.5 & 5 & 7.5 & 7.6 \\
	54 & polling \allowbreak @~15, 16 \allowbreak @~s1\_before\_s2 & 10.4 & 10.5 & 6.7 & 8.5 \\
	55 & csma \allowbreak @~3, 4 \allowbreak @~time\_max & 10 & 6.1 & 10 & 9.8 \\
	56 & mapk\_cascade \allowbreak @~4, 30 \allowbreak @~activated\_time & 15 & 5.4 & 2.7 & 11.4 \\
	57 & firewire \allowbreak @~true, 3, 800 \allowbreak @~time\_min & 12.4 & 11.2 & 12.5 & 12.7 \\
	58 & dpm \allowbreak @~6, 4, 5 \allowbreak @~TminQueuesFull & 10 & 126.3 & 8 & 13 \\
	59 & firewire \allowbreak @~true, 3, 600 \allowbreak @~time\_max & 12.9 & 6.8 & 12.9 & 13 \\
	60 & flexible-manufacturing \allowbreak @~21, 1 \allowbreak @~M3Fail\_E & 8.6 & 27.5 & 6.6 & 13.2 \\
	61 & firewire \allowbreak @~true, 3, 600 \allowbreak @~time\_sending & 12.6 & 6.4 & 12.7 & 13.4 \\
	62 & polling \allowbreak @~16, 16 \allowbreak @~s1\_before\_s2 & 24.1 & 26.2 & 15.6 & 20.1 \\
	63 & firewire \allowbreak @~true, 3, 800 \allowbreak @~time\_max & 22.1 & 11.6 & 22.1 & 22.5 \\
	64 & firewire \allowbreak @~true, 3, 800 \allowbreak @~time\_sending & 22.6 & 11.2 & 22.5 & 22.8 \\
	65 & philosophers \allowbreak @~16, 1 \allowbreak @~MinExpTimeDeadlock & 25 & 35.1 & 18.7 & 23.2 \\
	66 & dpm \allowbreak @~6, 6, 5 \allowbreak @~PminQueue1Full & 24.9 & 46.5 & 25.4 & 25 \\
	67 & eajs \allowbreak @~5, 250, 11 \allowbreak @~ExpUtil & 34.3 & 22.1 & 34.4 & 33.9 \\
	68 & consensus \allowbreak @~6, 2 \allowbreak @~steps\_min & 42.9 & 66.7 & 34.8 & 40.7 \\
	69 & polling \allowbreak @~17, 16 \allowbreak @~s1\_before\_s2 & 54.2 & 60.2 & 35.4 & 45.1 \\
	70 & mapk\_cascade \allowbreak @~5, 30 \allowbreak @~activated\_time & 68.1 & 40 & 11.6 & 51.7 \\
	71 & firewire \allowbreak @~true, 36, 200 \allowbreak @~time\_min & 61.5 & 54.2 & 61.9 & 65.7 \\
	72 & polling \allowbreak @~18, 16 \allowbreak @~s1\_before\_s2 & 122.2 & 142.9 & 80.6 & 101.8 \\
	73 & consensus \allowbreak @~6, 2 \allowbreak @~steps\_max & 119.8 & 103.6 & 57.6 & 112.9 \\
	74 & eajs \allowbreak @~6, 300, 13 \allowbreak @~ExpUtil & 112.1 & 70.4 & 112.9 & 114 \\
	75 & dpm \allowbreak @~8, 4, 5 \allowbreak @~PminQueue1Full & 117.4 & 206.1 & 119.9 & 118.4 \\
	76 & coupon \allowbreak @~9, 4, 5 \allowbreak @~exp\_draws & 120.3 & 83.3 & 118.7 & 127.6 \\
	77 & firewire \allowbreak @~true, 36, 200 \allowbreak @~time\_sending & 133.1 & 55.9 & 133.2 & 136.9 \\
	78 & firewire \allowbreak @~true, 36, 200 \allowbreak @~time\_max & 131.8 & 104.8 & 133.3 & 140.2 \\
	79 & mapk\_cascade \allowbreak @~6, 30 \allowbreak @~activated\_time & 241.6 & 178.9 & 40.2 & 177.2 \\
	80 & polling \allowbreak @~19, 16 \allowbreak @~s1\_before\_s2 & 270.6 & 331.6 & 179.4 & 225.9 \\
	81 & firewire \allowbreak @~true, 36, 400 \allowbreak @~time\_min & 406.7 & 348.7 & 410.2 & 434.4 \\
	82 & egl \allowbreak @~10, 2 \allowbreak @~messagesB & 446.1 & 378.1 & 442.2 & 455.3 \\
	83 & egl \allowbreak @~10, 2 \allowbreak @~messagesA & 442.9 & 375.9 & 452.3 & 457.1 \\
	\hline
	\hline
	& SUM & 3250.3 & 3092.8 & 2739.3 & 3167.2
\end{longtable}

\newcolumntype{R}{p{210pt}<{\raggedright\arraybackslash}}
\begin{longtable}{ |c|R||r|r|r|r| }
	\caption{Detailed experimental results measuring time in seconds for each algorithm over instance in Group 4.}
	\label{Table: Best Times}\\
	\hline
	\multicolumn{2}{|c||}{Instance} & \multicolumn{4}{c|}{Time (seconds)} \\
	\hline
	\# & Model \allowbreak @~Paramenters \allowbreak @~Property & IVI & OVI & SVI & GVI \\
	\hline
	\endfirsthead
	\multicolumn{4}{c}%
	{\tablename\ \thetable\ -- \textit{Continued from previous page}} \\
	\hline
	\multicolumn{2}{|c||}{Instance} & \multicolumn{4}{c|}{Time (seconds)} \\
	\hline
	\# & Model \allowbreak @~Paramenters \allowbreak @~Property & IVI & OVI & SVI & GVI \\
	\hline
	\endhead
	\hline \multicolumn{6}{r}{\textit{Continued on next page}} \\
	\endfoot
	\hline
	\endlastfoot
	1 & embedded \allowbreak @~2, 12 \allowbreak @~actuators & 0.6 & 0.1 & 0.2 & 0 \\
	2 & embedded \allowbreak @~2, 12 \allowbreak @~io & 0.5 & 0.2 & 0.2 & 0 \\
	3 & embedded \allowbreak @~2, 12 \allowbreak @~main & 0.6 & 0.1 & 0.2 & 0 \\
	4 & embedded \allowbreak @~2, 12 \allowbreak @~sensors & 0.5 & 0.1 & 0.2 & 0 \\
	5 & embedded \allowbreak @~3, 12 \allowbreak @~actuators & 0.8 & 0.2 & 0.3 & 0 \\
	6 & embedded \allowbreak @~3, 12 \allowbreak @~io & 0.8 & 0.2 & 0.3 & 0 \\
	7 & embedded \allowbreak @~3, 12 \allowbreak @~main & 0.8 & 0.2 & 0.3 & 0 \\
	8 & embedded \allowbreak @~3, 12 \allowbreak @~sensors & 0.7 & 0.2 & 0.3 & 0 \\
	9 & embedded \allowbreak @~4, 12 \allowbreak @~actuators & 1 & 0.2 & 0.4 & 0 \\
	10 & embedded \allowbreak @~4, 12 \allowbreak @~io & 1 & 0.3 & 0.4 & 0 \\
	11 & embedded \allowbreak @~4, 12 \allowbreak @~main & 1 & 0.2 & 0.4 & 0 \\
	12 & embedded \allowbreak @~4, 12 \allowbreak @~sensors & 0.9 & 0.2 & 0.4 & 0 \\
	13 & embedded \allowbreak @~5, 12 \allowbreak @~actuators & 1.2 & 0.2 & 0.4 & 0 \\
	14 & embedded \allowbreak @~5, 12 \allowbreak @~io & 1.2 & 0.3 & 0.5 & 0 \\
	15 & embedded \allowbreak @~5, 12 \allowbreak @~main & 1.3 & 0.2 & 0.4 & 0 \\
	16 & embedded \allowbreak @~5, 12 \allowbreak @~sensors & 1.1 & 0.2 & 0.4 & 0 \\
	17 & embedded \allowbreak @~6, 12 \allowbreak @~main & 1.4 & 0.2 & 0.5 & 0 \\
	18 & embedded \allowbreak @~7, 12 \allowbreak @~main & 1.6 & 0.3 & 0.5 & 0 \\
	19 & embedded \allowbreak @~8, 12 \allowbreak @~main & 1.8 & 0.3 & 0.6 & 0 \\
	20 & haddad-monmege \allowbreak @~20, 0.7 \allowbreak @~target & 3.8 & 0.3 & 2 & 0 \\
	21 & embedded \allowbreak @~2, 12 \allowbreak @~danger\_time & 0.6 & 0.1 & 0.2 & 0 \\
	22 & embedded \allowbreak @~2, 12 \allowbreak @~up\_time & 0.6 & 0.1 & 0.2 & 0 \\
	23 & embedded \allowbreak @~3, 12 \allowbreak @~danger\_time & 0.8 & 0.1 & 0.2 & 0 \\
	24 & embedded \allowbreak @~3, 12 \allowbreak @~up\_time & 0.8 & 0.1 & 0.2 & 0 \\
	25 & embedded \allowbreak @~4, 12 \allowbreak @~danger\_time & 1.1 & 0.2 & 0.3 & 0 \\
	26 & embedded \allowbreak @~4, 12 \allowbreak @~up\_time & 1 & 0.2 & 0.3 & 0 \\
	27 & embedded \allowbreak @~5, 12 \allowbreak @~danger\_time & 1.4 & 0.2 & 0.3 & 0 \\
	28 & embedded \allowbreak @~5, 12 \allowbreak @~up\_time & 1.2 & 0.2 & 0.3 & 0 \\
	29 & embedded \allowbreak @~6, 12 \allowbreak @~danger\_time & 1.7 & 0.2 & 0.4 & 0 \\
	30 & embedded \allowbreak @~6, 12 \allowbreak @~up\_time & 1.4 & 0.2 & 0.4 & 0 \\
	31 & embedded \allowbreak @~7, 12 \allowbreak @~danger\_time & 2 & 0.3 & 0.4 & 0 \\
	32 & embedded \allowbreak @~7, 12 \allowbreak @~up\_time & 1.6 & 0.2 & 0.4 & 0 \\
	33 & embedded \allowbreak @~8, 12 \allowbreak @~danger\_time & 2.4 & 0.3 & 0.5 & 0 \\
	34 & embedded \allowbreak @~8, 12 \allowbreak @~up\_time & 1.8 & 0.3 & 0.5 & 0 \\
	35 & haddad-monmege \allowbreak @~20, 0.7 \allowbreak @~exp\_steps & 4 & 0.3 & 1.5 & 0 \\
	36 & csma \allowbreak @~3, 2 \allowbreak @~some\_before & 0.2 & 0.2 & 0.2 & 0.1 \\
	37 & dpm \allowbreak @~4, 4, 25 \allowbreak @~PminQueuesFull & 0.1 & 0.2 & 0.1 & 0.1 \\
	38 & dpm \allowbreak @~4, 4, 5 \allowbreak @~PminQueuesFull & 0.1 & 0.2 & 0.1 & 0.1 \\
	39 & embedded \allowbreak @~6, 12 \allowbreak @~actuators & 1.4 & 0.2 & 0.5 & 0.1 \\
	40 & embedded \allowbreak @~6, 12 \allowbreak @~io & 1.4 & 0.3 & 0.6 & 0.1 \\
	41 & embedded \allowbreak @~6, 12 \allowbreak @~sensors & 1.2 & 0.2 & 0.5 & 0.1 \\
	42 & embedded \allowbreak @~7, 12 \allowbreak @~actuators & 1.5 & 0.3 & 0.6 & 0.1 \\
	43 & embedded \allowbreak @~7, 12 \allowbreak @~io & 1.6 & 0.4 & 0.7 & 0.1 \\
	44 & embedded \allowbreak @~7, 12 \allowbreak @~sensors & 1.4 & 0.3 & 0.6 & 0.1 \\
	45 & embedded \allowbreak @~8, 12 \allowbreak @~actuators & 1.7 & 0.3 & 0.6 & 0.1 \\
	46 & embedded \allowbreak @~8, 12 \allowbreak @~io & 1.7 & 0.4 & 0.7 & 0.1 \\
	47 & embedded \allowbreak @~8, 12 \allowbreak @~sensors & 1.5 & 0.3 & 0.6 & 0.1 \\
	48 & ftwc \allowbreak @~4, 5 \allowbreak @~TimeMax & 5 & 1.5 & 4.7 & 0.1 \\
	49 & ftwc \allowbreak @~4, 5 \allowbreak @~TimeMin & 3.7 & 1.5 & 3.4 & 0.1 \\
	50 & consensus \allowbreak @~2, 16 \allowbreak @~c2 & 0.2 & 4.4 & 0.2 & 0.2 \\
	51 & consensus \allowbreak @~4, 2 \allowbreak @~c2 & 0.2 & 0.5 & 0.2 & 0.2 \\
	52 & consensus \allowbreak @~4, 2 \allowbreak @~disagree & 0.2 & 0.6 & 0.2 & 0.2 \\
	53 & dpm \allowbreak @~4, 4, 25 \allowbreak @~TminQueuesFull & 0.2 & 0.9 & 0.2 & 0.2 \\
	54 & dpm \allowbreak @~4, 4, 5 \allowbreak @~TminQueuesFull & 0.2 & 1 & 0.2 & 0.2 \\
	55 & dpm \allowbreak @~4, 6, 100 \allowbreak @~PminQueue1Full & 0.3 & 0.4 & 0.3 & 0.3 \\
	56 & dpm \allowbreak @~4, 6, 100 \allowbreak @~PminQueuesFull & 0.3 & 0.7 & 0.4 & 0.3 \\
	57 & dpm \allowbreak @~4, 6, 25 \allowbreak @~PminQueue1Full & 0.3 & 0.4 & 0.3 & 0.3 \\
	58 & dpm \allowbreak @~4, 6, 25 \allowbreak @~PminQueuesFull & 0.3 & 0.7 & 0.4 & 0.3 \\
	59 & dpm \allowbreak @~4, 6, 50 \allowbreak @~PminQueue1Full & 0.3 & 0.4 & 0.3 & 0.3 \\
	60 & dpm \allowbreak @~4, 6, 50 \allowbreak @~PminQueuesFull & 0.3 & 0.7 & 0.4 & 0.3 \\
	61 & exploding-blocksworld \allowbreak @~5 \allowbreak @~goal & 0.4 & 0.3 & 0.3 & 0.3 \\
	62 & polling \allowbreak @~11, 16 \allowbreak @~s1\_before\_s2 & 0.4 & 0.4 & 0.3 & 0.3 \\
	63 & consensus \allowbreak @~4, 2 \allowbreak @~steps\_min & 0.3 & 0.4 & 0.3 & 0.3 \\
	64 & coupon \allowbreak @~7, 3, 5 \allowbreak @~collect\_all & 0.5 & 0.5 & 0.5 & 0.4 \\
	65 & echoring \allowbreak @~2 \allowbreak @~MinFailed & 0.4 & 0.5 & 0.4 & 0.4 \\
	66 & firewire \allowbreak @~true, 3, 200 \allowbreak @~elected & 0.4 & 0.5 & 0.4 & 0.4 \\
	67 & ftwc \allowbreak @~8, 5 \allowbreak @~TimeMax & 43.5 & 10.1 & 42 & 0.4 \\
	68 & ftwc \allowbreak @~8, 5 \allowbreak @~TimeMin & 28.5 & 10.5 & 27 & 0.4 \\
	69 & consensus \allowbreak @~4, 4 \allowbreak @~c2 & 0.9 & 8 & 1 & 0.7 \\
	70 & dpm \allowbreak @~4, 8, 100 \allowbreak @~PminQueue1Full & 0.8 & 1.4 & 0.8 & 0.8 \\
	71 & dpm \allowbreak @~4, 8, 25 \allowbreak @~PminQueue1Full & 0.8 & 1.4 & 0.8 & 0.8 \\
	72 & dpm \allowbreak @~4, 8, 5 \allowbreak @~PminQueue1Full & 0.8 & 1.4 & 0.8 & 0.8 \\
	73 & egl \allowbreak @~5, 8 \allowbreak @~messagesA & 0.8 & 1 & 0.8 & 0.8 \\
	74 & crowds \allowbreak @~4, 20 \allowbreak @~positive & 1 & 0.9 & 0.9 & 0.9 \\
	75 & dpm \allowbreak @~4, 8, 100 \allowbreak @~PminQueuesFull & 0.9 & 2.4 & 1 & 0.9 \\
	76 & dpm \allowbreak @~4, 8, 25 \allowbreak @~PminQueuesFull & 0.9 & 2.4 & 1 & 0.9 \\
	77 & dpm \allowbreak @~4, 8, 5 \allowbreak @~PminQueuesFull & 0.9 & 2.4 & 1 & 0.9 \\
	78 & consensus \allowbreak @~4, 4 \allowbreak @~disagree & 1.1 & 9.8 & 1.3 & 1.1 \\
	79 & dpm \allowbreak @~6, 4, 5 \allowbreak @~PminQueue1Full & 2.9 & 4.6 & 2.9 & 2.8 \\
	80 & dpm \allowbreak @~6, 4, 5 \allowbreak @~PminQueuesFull & 3.8 & 7.1 & 3.8 & 3.7 \\
	81 & consensus \allowbreak @~6, 2 \allowbreak @~c2 & 16.9 & 90.4 & 19.7 & 13.8 \\
	82 & consensus \allowbreak @~6, 2 \allowbreak @~disagree & 26 & 133.8 & 30.5 & 22.1 \\
	83 & dpm \allowbreak @~6, 6, 5 \allowbreak @~PminQueuesFull & 31 & 76.5 & 31.6 & 29.5 \\
	84 & dpm \allowbreak @~6, 8, 5 \allowbreak @~PminQueue1Full & 130 & 252.4 & 134.8 & 128.6 \\
	85 & dpm \allowbreak @~6, 8, 5 \allowbreak @~PminQueuesFull & 156.9 & 426.4 & 154 & 145.3 \\
	86 & dpm \allowbreak @~8, 4, 5 \allowbreak @~PminQueuesFull & 157 & 312 & 160.4 & 148.9 \\
	\hline
	\hline
	& SUM & 678.1 & 1381.7 & 650.3 & 510.4
	
\end{longtable}

\section{Comparative Grouping}
\label{Section: Comparative grouping}

In this section, we provide alternative distribution of instances in groups 3 and 4 focusing on the other VI-based approaches.
Note that, in \Cref{Section: Experiments (main text)} instances that do not belong to groups 1 and 2 are divided with respect to our algorithm GVI and \Cref{Figure: Walltimes} presents instances ordered by the time GVI takes. 
Below, we provide the corresponding presentation focusing on the other VI-based approaches.

The benchmark set contains 636 instances.
There are 162 instances where some of the algorithms considered timed out (at 600 seconds) or failed.
Omitting these 162 instances leaves a total of 474 instances that we analyze.
We grouped the instances depending on the focus variant as follows.
\begin{itemize}
	\item Group 1: instances where all algorithms are fast, i.e., they take at most 0.1 seconds (170 instances).
	\item Group 2: from the rest, those where the fastest and slowest algorithms are only at most 1.10 times of each other (135 instances).
	\item Group 3: from the rest, those where the focus variant is not the fastest.
	\item Group 4: all other instances not considered before.
\end{itemize}
In the rest of the section we present the concrete groups 3 and 4 for different focus variants.

\subsection{IVI}

Focusing on IVI, the new distribution of groups 3 and 4 is as follows.
In Group~3, the average and median speedups on the overall performance are as follows compared to IVI.
On average,
OVI is 0.91 times faster,
SVI is 1.17 times faster, and
GVI is 1.08 times faster.
With respect to the median,
OVI is 1.20 times faster,
SVI is 1.39 times faster, and
GVI is 2.25 times faster.
In Group~4, which is favorable for IVI, the average and median speedup on the overall performance of IVI are as follows.
On average, IVI is
1.87 times faster than OVI,
1.03 times faster than SVI, and
1.02 times faster than GVI.
With respect to the median, IVI is
2.34 times faster than OVI,
1.34 times faster than SVI, and
1.00 times faster than GVI.

\Cref{Figure: Walltimes IVI} shows the time (total execution time in the machine, not just CPU time) measured in seconds for every algorithm for each instance in the last two groups.
Instances are ordered by the time achieved by our algorithm.
Note that the $y$-axes are on a logarithmic scale.

\begin{figure}[H]
	\centering
	\includegraphics[width=\textwidth]{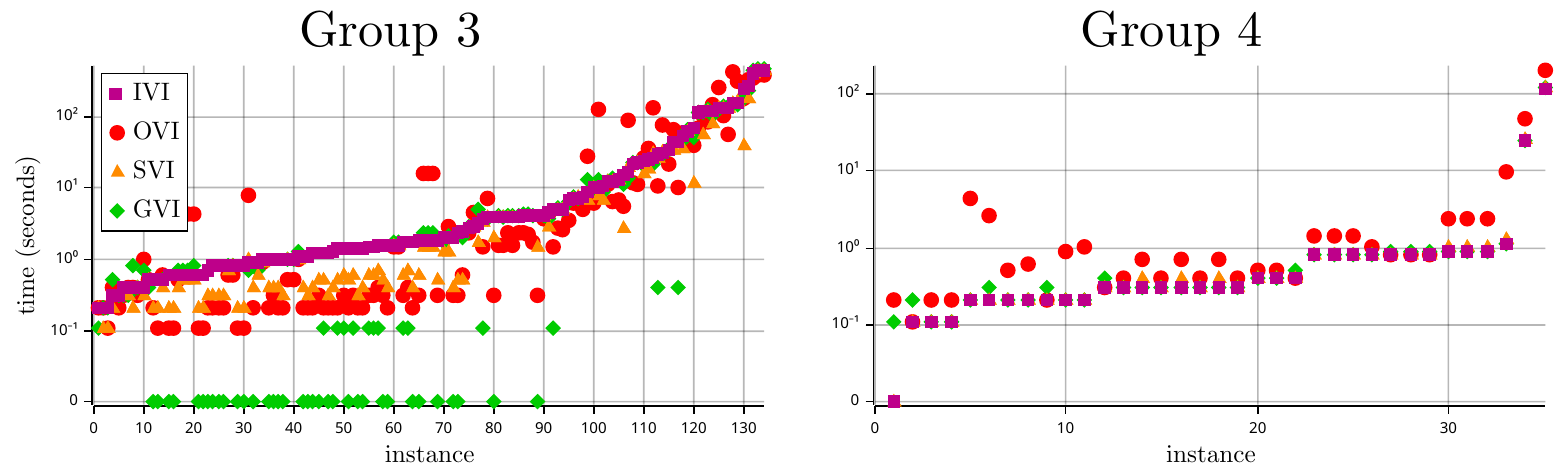}
	\caption{
		Time in seconds of all algorithms over instances not in groups 1 and 2 in increasing order according to IVI and displayed in logarithmic scale.
		Group 3 in this graph contains instances where IVI is not the fastest.
		Group 4 contains the rest of the instances, which is beneficial for IVI.
	}
	\label{Figure: Walltimes IVI}
\end{figure}

\subsection{OVI}

Focusing on OVI, the new distribution of groups 3 and 4 is as follows.
In Group~3, the average and median speedups on the overall performance are as follows compared to OVI.
On average,
IVI is 1.55 times faster,
SVI is 2.20 times faster, and
GVI is 1.89 times faster.
With respect to the median,
IVI is 0.42 times faster,
SVI is 1.00 times faster, and
GVI is 1.67 times faster.
In Group~4, which is favorable for OVI, the average and median speedup on the overall performance of OVI are as follows.
On average, OVI is
1.30 times faster than IVI,
1.30 times faster than SVI, and
1.35 times faster than GVI.
With respect to the median, OVI is
1.74 times faster than IVI,
1.70 times faster than SVI, and
1.79 times faster than GVI.

\Cref{Figure: Walltimes OVI} shows the time (total execution time in the machine, not just CPU time) measured in seconds for every algorithm for each instance in the last two groups.
Instances are ordered by the time achieved by our algorithm.
Note that the $y$-axes are on a logarithmic scale.

\begin{figure}[H]
	\centering
	\includegraphics[width=\textwidth]{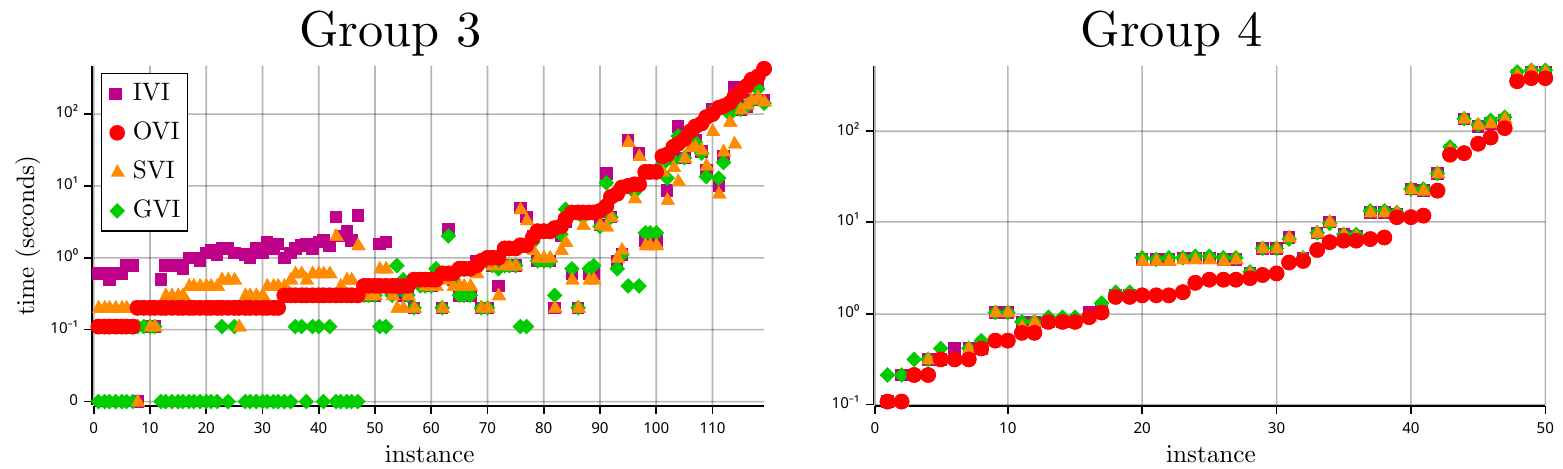}
	\caption{
		Time in seconds of all algorithms over instances not in groups 1 and 2 in increasing order according to OVI and displayed in logarithmic scale.
		Group 3 in this graph contains instances where OVI is not the fastest.
		Group 4 contains the rest of the instances, which is beneficial for OVI.
	}
	\label{Figure: Walltimes OVI}
\end{figure}

\subsection{SVI}

Focusing on SVI, the new distribution of groups 3 and 4 is as follows.
In Group~3, the average and median speedups on the overall performance are as follows compared to SVI.
On average,
IVI is 1.00 times faster,
OVI is 0.89 times faster, and
GVI is 1.03 times faster.
With respect to the median,
IVI is 0.59 times faster,
OVI is 1.43 times faster, and
GVI is 2.50 times faster.
In Group~4, which is favorable for SVI, the average and median speedup on the overall performance of SVI are as follows.
On average, SVI is
2.00 times faster than IVI,
2.39 times faster than OVI, and
1.69 times faster than GVI.
With respect to the median, SVI is
1.60 times faster than IVI,
2.00 times faster than OVI, and
1.60 times faster than GVI.

\Cref{Figure: Walltimes SVI} shows the time (total execution time in the machine, not just CPU time) measured in seconds for every algorithm for each instance in the last two groups.
Instances are ordered by the time achieved by our algorithm.
Note that the $y$-axes are on a logarithmic scale.

\begin{figure}[H]
	\centering
	\includegraphics[width=\textwidth]{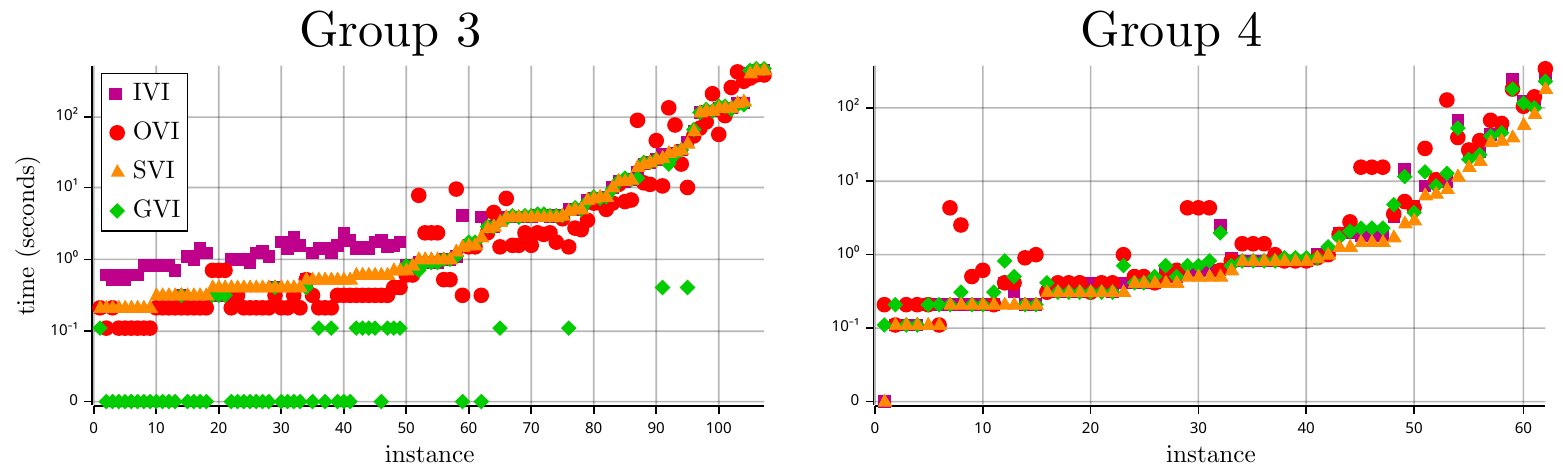}
	\caption{
		Time in seconds of all algorithms over instances not in groups 1 and 2 in increasing order according to SVI and displayed in logarithmic scale.
		Group 3 in this graph contains instances where SVI is not the fastest.
		Group 4 contains the rest of the instances, which is beneficial for SVI.
	}
	\label{Figure: Walltimes SVI}
\end{figure}

\subsection{GVI}

We present the results focusing on GVI, including average and median statistics.
In Group~3, the average and median speedups on the overall performance are as follows compared to GVI.
On average,
IVI is 0.98 times faster,
OVI is 1.03 times faster, and
SVI is 1.16 times faster.
With respect to the median,
IVI is 1.06 times faster,
OVI is 1.14 times faster, and
SVI is 1.08 times faster.
In Group~4, which is favorable for GVI, the average and median speedup on the overall performance of GVI are as follows.
On average, GVI is
1.33 times faster than IVI,
2.71 times faster than OVI, and
1.28 times faster than SVI.
With respect to the median, GVI is
10.00 times faster than IVI,
3.00 times faster than OVI, and
4.00 times faster than SVI.

\Cref{Figure: Walltimes} shows the time (total execution time in the machine, not just CPU time) measured in seconds for every algorithm for each instance in the last two groups.
Instances are ordered by the time achieved by our algorithm.
Note that the $y$-axes are on a logarithmic scale.

\subsection{Summary}

\Cref{Table: Comparative grouping 3} and \Cref{Table: Comparative grouping 4} shows all relevant quantities for groups 3 and 4 respectively.

\begin{table}
	\caption{Relative speedups of groups 3 varying the VI-based variant of focus. Larger values means that the focus variant is slower.}
	\label{Table: Comparative grouping 3}
	\centering
	\begin{tabular}{|l|l|l|l|l|l|l|l|l|}
		\hline
		\multirow{3}{*}{Focus variant} & \multicolumn{8}{c|}{Speedup of other variant} \\
		\cline{2-9}		
		& \multicolumn{4}{c|}{Average}
		& \multicolumn{4}{c|}{Median} \\
		\cline{2-9}
		&  IVI 
		& OVI
		& SVI
		& GVI 
		&  IVI 
		& OVI
		& SVI
		& GVI \\
		\hline
		IVI 
		& 1.00 
		& 1.15
		& 1.15
		& 0.82
		& 1.00 
		& 1.07
		& 1.89
		& 2.13  \\
		OVI 
		& 1.14 
		& 1.00
		& 1.47
		& 0.74
		& 1.50
		& 1.00
		& 3.00
		& 4.00 \\
		SVI 
		& 1.03
		& 1.21
		& 1.00
		& 0.91
		& 0.98
		& 1.90
		& 1.00
		& 0.98 \\
		GVI 
		& 1.32
		& 1.39
		& 1.45
		& 1.00
		& 1.70
		& 1.07
		& 2.13
		& 1.00 \\
		\hline
	\end{tabular}
\end{table}

\begin{table}
	\caption{Relative speedups of groups 4 varying the VI-based variant of focus. Larger values means that the focus variant is faster.}
	\label{Table: Comparative grouping 4}
	\centering
	\begin{tabular}{|l|l|l|l|l|l|l|l|l|}
		\hline
		\multirow{3}{*}{Focus variant} & \multicolumn{8}{c|}{Speedup over other variant} \\
		\cline{2-9}		
		& \multicolumn{4}{c|}{Average}
		& \multicolumn{4}{c|}{Median} \\
		\cline{2-9}
		&  IVI 
		& OVI
		& SVI
		& GVI
		&  IVI 
		& OVI
		& SVI
		& GVI \\
		\hline
		IVI 
		& 1.00 
		& 1.51
		& 1.26
		& 1.80
		& 1.00
		& 1.00
		& 1.00
		& 1.50 \\
		OVI 
		& 1.27
		& 1.00
		& 1.28
		& 1.39
		& 1.19
		& 1.00
		& 1.00
		& 1.37 \\
		SVI 
		& 1.59
		& 1.76
		& 1.00
		& 2.59 
		& 2.00
		& 3.34
		& 1.00
		& 1.67 \\
		GVI 
		& 18.92
		& 5.86
		& 15.55
		& 1.00
		& 10.00
		& 3.00
		& 4.00
		& 1.00 \\
		\hline
	\end{tabular}
\end{table}

\end{document}